%% file: manuscript.tex
\documentclass[10pt,journal,cspaper,compsoc]{IEEEtran}

\usepackage{float}
\usepackage{graphicx}
\usepackage{amsmath}
\usepackage{amsthm}
\usepackage{amssymb}
\usepackage{booktabs}
\usepackage[lined]{algorithm2e}

\input{macros}

\usepackage[pagebackref=true,breaklinks=true,letterpaper=true,colorlinks,bookmarks=false]{hyperref}

\begin{document}

\title{Sparse Representation for 3D Shape Estimation: A Convex Relaxation Approach}

\author{Xiaowei~Zhou,~\IEEEmembership{}
        Menglong Zhu, ~\IEEEmembership{}
        Spyridon Leonardos,~\IEEEmembership{}
        and~Kostas~Daniilidis,~\IEEEmembership{Fellow,~IEEE}
\IEEEcompsocitemizethanks{\IEEEcompsocthanksitem The authors are with Computer and Information Science Department and GRASP Laboratory, University of Pennsylvania, Philadelphia,
PA, 19104.\protect\\
E-mail: xiaowz@seas.upenn.edu}
\thanks{To appear in IEEE Transactions on Pattern Analysis and Machine Intelligence. \url{http://ieeexplore.ieee.org/document/7558185/}}}

\markboth{}%
{Zhou \MakeLowercase{\textit{et al.}}: Bare Advanced Demo of IEEEtran.cls for Journals}

\IEEEtitleabstractindextext{%
\begin{abstract}

We investigate the problem of estimating the 3D shape of an object defined by a set of 3D landmarks, given their 2D correspondences in a single image. A successful approach to alleviating the reconstruction ambiguity is the 3D deformable shape model and a sparse representation is often used to capture complex shape variability. But the model inference is still challenging due to the nonconvexity in the joint optimization of shape and viewpoint. In contrast to prior work that relies on an alternating scheme whose solution depends on initialization, we propose a convex approach to addressing this challenge and develop an efficient algorithm to solve the proposed convex program. We further propose a robust model to handle gross errors in the 2D correspondences. We demonstrate the exact recovery property of the proposed method, the advantage compared to several nonconvex baselines and the applicability to recover 3D human poses and car models from single images.
\end{abstract}
}

\maketitle

\section{Introduction}

\begin{figure*}
  \centering
  \includegraphics[width=0.9\linewidth]{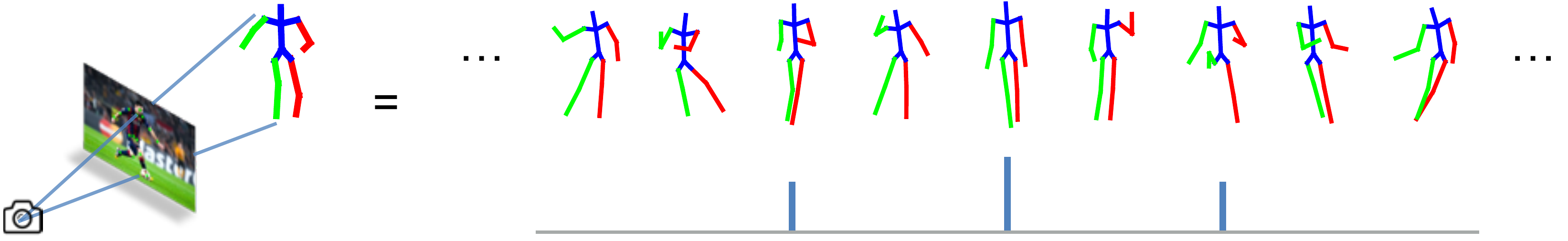}\\
  \caption{Illustration of the problem studied in this paper. The unknown 3D model is defined by a set of landmarks and assumed to be a linear combination of some predefined basis shapes with sparse coefficients. Given the 2D correspondences of the landmarks in a single image, the computational problem is to simultaneously estimate the coefficients of the sparse representation as well as the viewpoint of the camera.  }\label{fig:illustration}
\end{figure*}

Recognizing 3D objects from 2D images is a central problem in computer vision. Past years have witnessed an emerging trend towards analyzing 3D geometry of objects including shapes and poses instead of merely providing bounding boxes \cite{xiang2012estimating,aubry2014seeing}. The 3D geometric reasoning can not only provide richer information about the scene for subsequent high-level tasks such as scene understanding, augmented reality and human computer interaction, but also improve 2D recognition performances \cite{fidler20123d,simo2013joint}. 3D reconstruction has been a well studied problem and there have been many practically applicable techniques such as structure from motion, multi-view stereo and depth sensors, but these techniques are limited in some scenarios. For example, a large number of acquisitions from multiple views are often required in order to obtain a complete 3D model, which is not preferred in some real-time applications; the depth sensors in general cannot work outdoor and have a limited sensing range; and reconstructing a dynamic scene is still an open problem. In this paper, we aim to investigate the possibility of estimating the 3D shape of an object from a single 2D image, which is complementary to the aforementioned techniques and may potentially address the above issues.

Estimating the 3D geometry of an object from a single view is an ill-posed problem. But it is a possible task for a human observer, since human can leverage visual memory of object shapes. Inspired by this idea, much effort has been made towards 3D model-based analysis leveraging the increasing availability of online 3D model databases, such as the Google 3D warehouse that includes millions of CAD models of various objects and many publicly available shape scans and motion capture data that model the 3D shape and pose of human.

To address intra-class variability and nonrigid deformation and avoid exhaustively enumerating all possibilities, many previous works, e.g. \cite{hejrati2012analyzing,zia2013detailed}, adopted a 3D deformable model originated from the ``active shape model" \cite{cootes1995active} to represent shapes, where each shape is defined by a set of ordered landmarks and the one to be estimated is assumed as a linear combination of some basis shapes usually obtained from principal component analysis (PCA). For human poses, the sparse representation was proposed to handle highly articulated deformation of human bodies which cannot be captured by PCA \cite{ramakrishna2012reconstructing,wang2014robust}. In model inference, the 3D deformable model is matched to the landmarks annotated or detected in images, and the problem is reduced to a 3D-to-2D shape fitting problem where the shape (weights of the linear combination) and the viewpoint (camera extrinsic parameters) have to be estimated simultaneously. \refFig{fig:illustration} gives an illustration of the problem.

While this approach has achieved promising results in various applications, the model inference is still a challenging problem, since the subproblems of shape and viewpoint estimation are coupled: the viewpoint needs to be known in order to fit the 3D model to 2D, and inversely, the 3D model needs to be known in order to estimate the viewpoint. The joint estimation of shape and viewpoint results in a nonconvex optimization problem, and the orthogonality constraint on the camera rotation makes the problem even more complicated. Previous methods often rely on an alternating scheme to alternately update the shape and viewpoint parameters, which has no guarantee for global convergence and is sensitive to initialization. As mentioned in many previous works, e.g. \cite{ramakrishna2012reconstructing,hejrati2012analyzing}, most of the failed cases were attributed to bad initialization. Some heuristics have been proposed to address this issue, such as initializing multiple times \cite{wang2014robust} or using a viewpoint-aware detector \cite{zia2013detailed}. But there is still no guarantee for global optimality.

In this paper, we propose a convex relaxation approach to addressing the aforementioned issue. We use an augmented shape-space model, in which a shape is represented as a linear combination of rotatable basis shapes. This model gives a linear representation of both intrinsic shape deformation and extrinsic viewpoint changes. Next, we use the convex relaxation of the orthogonality constraint to convert the entire problem into a spectral-norm regularized least squares problem, which is a convex program. Then, we develop an efficient algorithm to solve the proposed convex program based on the alternating direction method of multipliers (ADMM). To achieve fast computation, we derive a closed-form solution to compute the proximal operator of the spectral norm. Furthermore, we extend the proposed model to handle outliers in the input 2D correspondences. This work extends its earlier version \cite{zhou20143d} with the extensions including postprocessing for rotation synchronization, outlier modeling, more experiments and real examples with learned detectors.

The remainder of this paper is organized as follows. We summarize related work in \refSec{sec:related}, formulate the problem in \refSec{sec:formulation}, introduce the proposed method and technical details in \refSec{sec:methods}, report empirical results in \refSec{sec:experiments}, and finally conclude the paper in \refSec{sec:discussion}. The code is available at \url{http://cis.upenn.edu/~xiaowz/shapeconvex.html}.

\section{Related work}\label{sec:related}

The most related work includes the ones that fit a 3D deformable model to the features in a 2D image. A popular application is human face modeling for various tasks such as recognition \cite{blanz2003face}, feature tracking \cite{gu20063d} and facial animation \cite{cao20133d}. Recently, there has been an increased number of works on 3D object modeling. For example, Hejrati and Ramanan \cite{hejrati2012analyzing} used the deformable model for 3D car modeling. They produced 2D landmarks by a variant of deformable part models \cite{felzenszwalb2008discriminatively} and then lifted the 2D model to 3D. Lin et al. \cite{lin2014jointly} proposed a method for joint 3D model fitting and fine-grained classification for cars. Zia \cite{zia2013detailed} et al. developed a probabilistic framework to simultaneously localize 2D landmarks and recover 3D object models.

While the conventional active shape model performs well to describe the shape variability of rigid objects, it can hardly handle the structural variability of nonrigid objects such as human bodies. To address this issue, Ramakrishna et al. \cite{ramakrishna2012reconstructing} proposed a sparse representation based approach to reconstructing 3D human pose from annotated joints in a still image. Wang et al. \cite{wang2014robust} adopted a 2D human pose detector to automatically locate the joints and used a robust estimator to handle inaccurate joint locations. Fan et al. \cite{fan2014pose} proposed to improve the performance of \cite{ramakrishna2012reconstructing} by enforcing locality when building the pose dictionary. Zhou and De La Torre \cite{zhou2014sptio} formulated human pose estimation as a matching problem, where a learned spatio-temporal pose model was matched to point trajectories extracted from a video. Akhter and Black \cite{akhter2015pose} integrated a joint-angle constraint into the sparse representation to reduce the possibility of invalid reconstruction.

Recently, there was growing interest in reconstructing category-specific object models from a collection of single images of different instances with the same category \cite{cashman2013shape,vicente2014reconstructing,carreira2015virtual,kar2015category}. The deformable shape model was adopted and fitted to either visual hulls in the pre-segmented images or landmarks annotated in the images, and at the same time the basis shapes were also learned from images during the model inference. The problem studied in this paper can be regarded as a subproblem in this process where the basis shapes are given.

A common component or an intermediate step in the works mentioned above is to fit a 3D deformable model to 2D correspondences. As we mentioned in the introduction, these works relied on nonconvex optimization, which may be sensitive to initialization. The convex formulation proposed in this paper can potentially serve as a building block or provide a good initialization to improve the performance of the existing methods.

Another line of work tried to solve single-image reconstruction by finding the nearest neighbor in the shape collection followed by a refinement step \cite{su2014estimating,huang2015single}. The initial shape and viewpoint were found by enumerating all instances and viewpoints and comparing the test image with the rendered one. The initial estimate was refined by optimizing both the viewpoint and nonrigid deformation according to image contours. This approach produces very detailed reconstruction and is applicable to a wide range of object categories, but it is computationally expensive and requires accurate image segmentation.

There are also many alternative approaches for 3D human pose recovery from single images such as using a known articulated skeleton \cite{taylor2000reconstruction,guan2009estimating,leonardos2016articulated}, probabilistic graphical models \cite{sigal2006predicting,andriluka2010monocular}, explicit regression \cite{elgammal2004inferring,agarwal2006recovering}, to name a few. But these approaches are customized for human pose modeling and not straightforwardly generalizable for other objects.

Our work is also closely related to nonrigid structure from motion (NRSfM), where a deformable shape is recovered from multi-frame 2D-2D correspondences. The low-rank shape-space model has been frequently used in NRSfM, but the basis shapes are unknown. The joint estimation of shape/pose variables and basis shapes was typically solved via matrix factorization followed by metric rectification \cite{bregler2000recovering,xiao2006closed}. In some recent works, iterative algorithms were employed for better precision \cite{paladini2012optimal,del2012bilinear} or sequential processing \cite{agudo2014good}, and the problem studied in this paper is analogous to the step of fixing basis shapes and updating the remaining variables in those works.

\section{Problem statement}\label{sec:formulation}

The problem studied in this paper can be described by the following linear system:
\begin{align}\label{eq:basic}
\bfW = \Pi\bfS,
\end{align}
where $\bfS\in\RR{3}{p}$ denotes the unknown 3D shape, which is represented by 3D locations of $p$ points. $\bfW\in\RR{2}{p}$ denotes their projections in a 2D image. $\Pi$ is the camera calibration matrix. To simplify the problem, the weak-perspective camera model is usually used, which is a good approximation when the object depth is much smaller than the distance from the camera. With this assumption, the calibration matrix has the following simple form:
\begin{align}\label{eq:calibration}
\Pi = \begin{bmatrix}
            s & 0 & 0 \\
            0 & s & 0 \\
            \end{bmatrix},
\end{align}
where $s$ is a scalar depending on the focal length and the distance to the object.

There are always more variables than equations in \refEq{eq:basic}. To make the problem well-posed, a widely-used assumption is that the unknown shape can be represented as a linear combination of predefined basis shapes, which is originated from the active shape model \cite{cootes1995active}:
\begin{align}\label{eq:shape-model-original}
    \bfS = \sum_{i=1}^{k} c_i\bfB_i,
\end{align}
where $\bfB_i\in\RR{3}{p}$ for $i\in[1,k]$ represents a basis shape and $c_i$ its weight in the combination. Thus, the reconstruction problem is reduced to a problem of estimating several coefficients by fitting the model \refEq{eq:shape-model-original} to the landmarks in an image, which greatly reduces the number of unknowns.

Since the basis shapes are predefined, the relative rotation and translation between the camera frame and the coordinates that define the basis shapes need to be taken into account, and the 3D-2D projection is depicted by:
\begin{align}\label{eq:2d3dcorresp}
\bfW = \Pi\left(\bfR\sum_{i=1}^{k} c_i\bfB_i + \bfT\bfone^T\right),
\end{align}
where $\bfR\in\RR{3}{3}$ and $\bfT\in\R{3}$ correspond to the rotation matrix and the translation vector, respectively. $\bfR$ should be in the special orthogonal group
\begin{align}
SO(3) = \{\bfR\in\RR{3}{3}|\bfR^T\bfR=\bfI_3,\det{\bfR}=1\}.
\end{align}

Equation \refEq{eq:2d3dcorresp} can be further simplified as
\begin{align}\label{eq:bilinear}
\bfW = \bar{\bfR}\sum_{i=1}^{k} c_i\bfB_i,
\end{align}
where $\bar{\bfR}\in\RR{2}{3}$ denotes the first two rows of the rotation matrix, and the translation $\bfT$ has been eliminated by centralizing the data, i.e. subtracting each row of $\bfW$ and $\bfB$ by its mean. Note that the scalar $s$ in the calibration matrix has been absorbed into $c_1,\cdots,c_k$.

In the active shape model \cite{cootes1995active}, the principal components of training samples are used as the basis shapes, which assumes that all shapes lie in a low-dimensional linear space. In more recent work, e.g. \cite{ramakrishna2012reconstructing,zhang2011sparse,zhu2010model,zhu2014complex}, it has been shown that the low-dimensional linear space is insufficient to model complex shape variation, e.g., human poses, and a promising approach is to use an over-complete dictionary and represent an unknown shape as a sparse combination of atoms in the dictionary. Such a sparse representation implicitly encodes the assumption that the unknown shape should lie in a union of subspaces that approximates a nonlinear shape manifold. The representabilities of PCA and sparse representation for 3D human pose modeling will be empirically compared in \refSec{sec:human}.

Based on the sparse representation of shapes, the following type of optimization problem is often considered to estimate an unknown shape:
\begin{align}\label{eq:originalnoisy}
    \min_{\bfc,\bar{\bfR}} ~~ & \half \left\| \bfW - \bar{\bfR}\sum_{i=1}^{k} c_i\bfB_i \right\|_F^2 + \alpha \|\bfc\|_1, \nonumber \\
    \st ~~ & \bar{\bfR}\bar{\bfR}^T = \bfI_2,
\end{align}
where $\bfc=[c_1,\cdots,c_k]^T$ and $\|\bfc\|_1$ represents the $\ell_1$ norm of $\bfc$, which is the convex surrogate of the cardinality. $\|\cdot\|_F$ denotes the Frobenius norm of a matrix. The terms in the loss function of \refEq{eq:originalnoisy} correspond to the reprojection error and the sparsity of representation, respectively.

The optimization in \refEq{eq:originalnoisy} is nonconvex and there is an orthogonality constraint. A commonly-used strategy to solve the optimization is the alternating minimization scheme, in which two steps are alternated: (1) fixing $\bar{\bfR}$ and updating $\bfc$ by any $\ell_1$ minimization solver, e.g., the methods summarized in \cite{bach2012optimization}; (2) fixing $\bfc$ and updating $\bar{\bfR}$ using certain rotation representations such as the quaternions, the exponential map or a manifold representation. In some previous works \cite{ramakrishna2012reconstructing,akhter2015pose}, $\bar{\bfR}$ is updated by the singular value decomposition (SVD), and it is worth noting that this method can only yield an approximate solution since $\bar{\bfR}$ is not a square matrix and $\bar{\bfR}^T\bar{\bfR}\neq \bfI$. Generally no closed-form solution exists in this case \cite{edelman1998geometry}.

The alternating algorithm is summarized in \refAlg{alg:altern}. Due to the nonconvexity of \refEq{eq:originalnoisy}, \refAlg{alg:altern} may get stuck at local minima when initialization is far away from the true solution.

\begin{algorithm}\label{alg:altern}
\LinesNumbered
\caption{Alternating minimization to solve \refEq{eq:originalnoisy}.}
\KwIn{$\bfW$}
\KwOut{$\bfc$ and $\bar\bfR$.}
\vspace{0.5em}
initialize $\bfc$ and $\bar\bfR$\;
\While{not converged}{
update $\bfc$ using any $\ell_1$ minimization solver\;
update $\bar\bfR$ using SVD or any local optimization solver over $SO(3)$\;
}
\vspace{0.5em}
\end{algorithm}

\section{Proposed methods}\label{sec:methods}

\subsection{Convex relaxation}\label{sec:relaxation}

We propose to use the following shape-space model:
\begin{align}\label{eq:shape-model-relaxed}
\bfS = \sum_{i=1}^{k} c_i\bfR_i\bfB_i,
\end{align}
in which there is a rotation for each basis shape. The model in \refEq{eq:shape-model-relaxed} implicitly accounts for the viewpoint variability and the projected 2D model is
\begin{align}\label{eq:new2dmodel}
\bfW = \Pi\sum_{i=1}^{k} c_i\bfR_i\bfB_i = \sum_{i=1}^{k} \bfM_i\bfB_i,
\end{align}
where $\bfM_i\in\RR{2}{3}$ is the product of $c_i$ and the first two rows of $\bfR_i$, which satisfies
\begin{align}\label{eq:orthogonality}
    \bfM_i\bfM_i^T = c_i^2\bfI_2.
\end{align}

The motivation of using the model in \refEq{eq:new2dmodel} is to achieve a linear representation of shape variability in 2D, such that we can get rid of the bilinear form in \refEq{eq:bilinear}, which is a necessary step towards a convex formulation.

The model in \refEq{eq:new2dmodel} is equivalent to the affine-shape model in literature \cite{blake2000active,xiao2004real}, which uses an augmented linear space to represent the shape variation in 2D caused by both intrinsic shape deformation and extrinsic viewpoint changes. This representation also appears in most NRSfM literature, e.g. \cite{bregler2000recovering,paladini2012optimal}. As mentioned in \cite{xiao2004real}, the augmented linear space can represent any 2D shape produced by the 3D shape model projected into the image plane, but the increased degree of freedom may result in invalid shapes. In this work, we try to reduce the possibility of invalid cases by enforcing the orthogonality constraint on $\bfM_i$s and the sparsity constraint on the number of activated basis shapes. We will show that these constraints can be conveniently imposed by minimizing a convex regularizer.

Next, replacing the orthogonality constraint in \refEq{eq:orthogonality} by its convex counterpart is considered. The following lemma has been proven in literature \cite[Section 3.4]{journee2010generalized}:

\begin{lemma}\label{lemma1}
The convex hull of the Stiefel manifold $\mathcal{Q}=\left\{\bfX\in\RR{m}{n} | ~\bfX^T\bfX = \bfI_n\right\}$ equals the unit spectral-norm ball $\conv{\mathcal{Q}}=\left\{\bfX\in\RR{m}{n} | ~~\|\bfX\|_2 \leq 1 \right\}$. $\|\bfX\|_2$ denotes the spectral norm (a.k.a. the induced 2-norm) of a matrix $\bfX$, which is defined as the largest singular value of $\bfX$.
\end{lemma}
Based on \refLemma{lemma1}, we have the following proposition:
\begin{proposition}
Given a scalar $s$, the convex hull of $\mathcal{S}=\left\{\bfY\in\RR{m}{n} | ~\bfY^T\bfY = s^2\bfI_n\right\}$ equals the spectral-norm ball with a radius of $|s|$: $\conv{\mathcal{S}}=\left\{\bfY\in\RR{m}{n} | ~~\|\bfY\|_2 \leq |s| \right\}$.
\end{proposition}
\begin{proof}
Since $\mathcal{S}=\left\{\bfY ~ | ~ \bfY=|s|\bfX, \bfX\in\mathcal{Q} \right\}$, we have
\begin{align}
\conv{\mathcal{S}} &= \left\{\sum_{i=1}^{k}\theta_i\bfY_i ~~ | ~~\bfY_i\in\mathcal{S}, \theta_i\geq 0, \sum_{i=1}^{k}\theta=1 \right\} \nonumber \\
&= \left\{\sum_{i=1}^{k}\theta_i|s|\bfX_i ~~ | ~~\bfX_i\in\mathcal{Q}, \theta_i\geq 0, \sum_{i=1}^{k}\theta=1 \right\} \nonumber \\
&= |s|\cdot\conv{\mathcal{Q}}. \nonumber
\end{align}
\end{proof}

Consequently, the tightest convex relaxation to the constraint in \refEq{eq:orthogonality} is given by $\|\bfM_i\|_2 \leq |c_i|$.

Finally, with the modified shape model, the relaxed orthogonality constraint and the assumption of sparse representation, the following optimization is proposed for shape recovery under the noiseless case:
\begin{align}\label{eq:originalnoiseless}
    \min_{c_1,\cdots,c_k,\bfM_1,\cdots,\bfM_k}~ & \sum_{i=1}^{k}|c_i|, \nonumber \\
    \st ~~~~~~~~~~ & \bfW = \sum_{i=1}^{k} \bfM_i\bfB_i, \nonumber \\
    & \|\bfM_i\|_2 \leq |c_i|, ~ \forall i\in[1,k].
\end{align}
In \refEq{eq:originalnoiseless}, reducing $|c_i|$ can decrease the objective and the only constraint on $|c_i|$ is that $\|\bfM_i\|_2 \leq |c_i|$. Therefore, $|c_i|$ is equal to $\|\bfM_i\|_2$ when an optimum is attained. Consequently, \refEq{eq:originalnoiseless} is equivalent to the following problem:
\begin{align}\label{eq:finalnoiseless}
    \min_{\bfM_1,\cdots,\bfM_k}~ & \sum_{i=1}^{k}\|\bfM_i\|_2, \nonumber \\
    \st ~~~~ & \bfW = \sum_{i=1}^{k} \bfM_i\bfB_i.
\end{align}
The formulation in \refEq{eq:finalnoiseless} is a linear inverse problem, where a set of orthogonal matrices\footnote{For simplicity, we call a matrix $\bfX$ as orthogonal matrix if it satisfies $\bfX^T\bfX=s\bfI$ or $\bfX\bfX^T=s\bfI$ where $s$ is a scalar.} are estimated by minimizing their spectral norms. Interestingly, the conditions for exact recovery using such a convex program has been theoretically analyzed in \cite{chandrasekaran2012convex}. Numerical results will be presented in \refSec{sec:simulation} to demonstrate the exact recovery property.

To consider noise in real applications, the following formulation is proposed:
\begin{align}\label{eq:finalnoisy}
    \min_{\bfM_1,\cdots,\bfM_k}~ & \half \left\| \bfW - \sum_{i=1}^{k} \bfM_i\bfB_i \right\|_{F}^2 + \alpha \sum_{i=1}^{k}\|\bfM_i\|_2.
\end{align}
The problem \refEq{eq:finalnoisy} is the final formulation, which is a penalized least-squares problem. We have following remarks:
\begin{itemize}
\item[1.] The problem in \refEq{eq:finalnoisy} is convex programming, which can be solved globally. We will provide an efficient algorithm to solve it in \refSec{sec:alg}.
\item[2.] Notice that $\|\cdot\|_2$ in the above formulations denotes the spectral norm of a matrix. As we will show in \refSec{sec:proximal}, minimizing the spectral norm of a matrix is equivalent to minimizing the $\ell_{\infty}$-norm of the vector of singular values, which will simultaneously shrink the norm of the matrix towards zero and enforce its singular values to be equal. Therefore, by spectral-norm minimization, we can not only minimize the number of activated basis shapes but also enforce each transformation matrix $\bfM_i$ to be orthogonal (an orthogonal matrix has equal singular values).
\item[3.] For the cases with missing or invisible landmarks, the model parameters can be estimated by minimizing reprojection errors of observed landmarks, i.e., including a binary weight matrix in the first term of \refEq{eq:finalnoisy}. The unobserved landmarks can be hallucinated from the reconstructed model as their locations are known on the basis shapes.
\end{itemize}

\subsection{Proximal operator of the spectral norm}\label{sec:proximal}

Before providing the specific algorithm to solve \refEq{eq:finalnoisy}, we first prove the following proposition, which will serve as an important building block in our algorithm.
\begin{proposition}\label{prop:prox2norm}
The solution to the following problem
\begin{align}\label{eq:prox-2norm}
\min_{\bfX} ~ \half \|\bfA-\bfX\|_F^2 + \lambda \|\bfX\|_2
\end{align}
is given by $\bfX^*=\mathcal{D}_{\lambda}(\bfA)$, where
\begin{align}\label{eq:proximal-operator}
\mathcal{D}_{\lambda}(\bfA) &= \bfU_A~\diag\left[\bfsigma_A - \lambda\proj_{\ell_1}(\bfsigma_A/\lambda)\right]~\bfV_A^T,
\end{align}
$\bfU_A$, $\bfV_A$ and $\bfsigma_A$ denote the left singular vectors, right singular vectors and the singular values of $\bfA$, respectively. $\proj_{\ell_1}$ is the projection of a vector to the unit $\ell_1$-norm ball.
\end{proposition}

\begin{proof}
The problem in \refEq{eq:prox-2norm} is a proximal problem \cite{parikh2013proximal}. The proximal problem associated with a function $F$ is defined as
\begin{align}
\prox_{\lambda F}(\bfA) = \arg\min_{\bfX} \half \|\bfA-\bfX\|_F^2 + \lambda F(\bfX),
\end{align}
with the solution denoted by $\prox_{\lambda F}(\bfA)$ and named the proximal operator of $F$.

For the problem \refEq{eq:prox-2norm}, $F(\bfX)=\|\bfX\|_2=\|\bfsigma_X\|_{\infty}$, where $\|\cdot\|_{\infty}$ means the $\ell_{\infty}$ norm. It says that $F$ is a spectral function operated on singular values of a matrix. Based on the property of spectral functions \cite[Section 6.7.2]{parikh2013proximal}, we have
\begin{align}
\prox_{\lambda F}(\bfA) = \bfU_A~\diag\left[\prox_{\lambda f}(\bfsigma_A)\right]~\bfV_A^T,
\end{align}
where $f$ is the $\ell_{\infty}$-norm. The proximal operator of the $\ell_{\infty}$-norm can be computed by Moreau decomposition \cite[Section 6.5]{parikh2013proximal}:
\begin{align}
\prox_{\lambda f}(\bfsigma_A) = \bfsigma_A - \lambda\proj_{\ell_1}(\bfsigma_A/\lambda),
\end{align}
given that the $\ell_1$-norm is the dual norm of the $\ell_{\infty}$-norm.
\end{proof}

The solution for the case with two singular values is illustrated in \refFig{fig:proximal}, where $[\Delta_1,\Delta_2]^T$ corresponds to $\bfsigma_A - \lambda\proj_{\ell_1}(\bfsigma_A/\lambda)$ in \refEq{eq:proximal-operator}. From the illustration we have two observations: (1) If the projection is on the edge of the $\ell_1$-norm ball, then $\Delta_1=\Delta_2$ and $\bfX^*$ turns out to be orthogonal; (2) If $\bfsigma_A$ is small and lies inside the $\ell_1$-norm ball, then the projection will be itself, $\Delta_1=\Delta_2=0$, and $\bfX^*$ turns out to be all-zero. These two observations illustrate the effects of spectral-norm regularization in \refEq{eq:finalnoisy}: (1) Enforcing $\bfM_i$s to be orthogonal; (2) Pruning small $\bfM_i$s to achieve a sparse representation.

\begin{figure}
  \centering
  \includegraphics[width=0.7\linewidth]{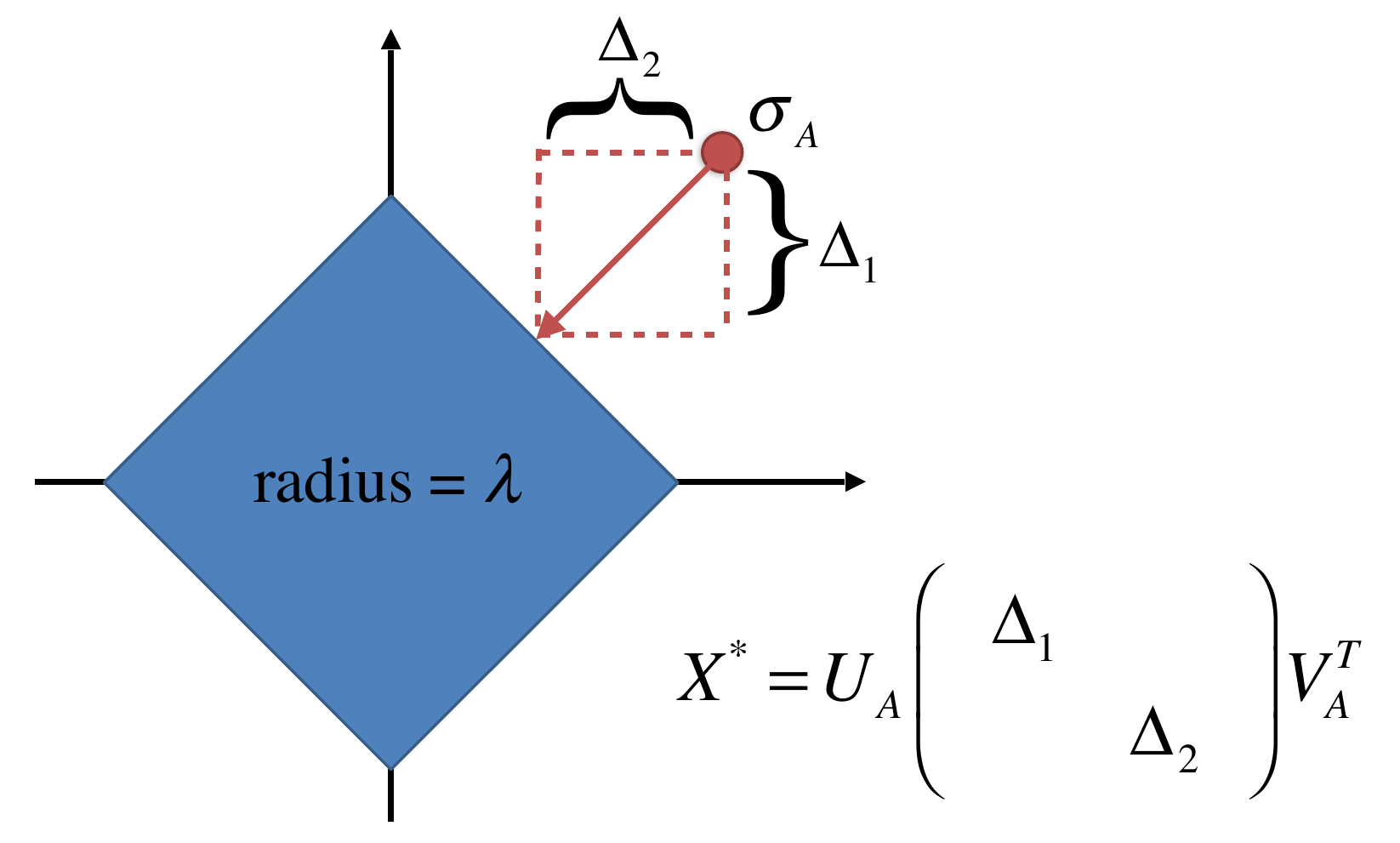}\\
  \caption{Illustration of the proximal operator of the spectral norm.}\label{fig:proximal}
\end{figure}

\subsection{Optimization}\label{sec:alg}

The algorithm to solve \refEq{eq:finalnoisy} is presented here. The noiseless case \refEq{eq:finalnoiseless} can be solved similarly.

The proposed algorithm is based on ADMM \cite{boyd2010distributed} and the proximal operator of the spectral norm derived in \refSec{sec:proximal}. An auxiliary variable $\bfZ$ is introduced and \refEq{eq:finalnoisy} is rewritten as
\begin{align}\label{eq:admm-noisy}
    \min_{\widetilde{\bfM},{\bfZ}}~ & \half \left\| \bfW - \bfZ\widetilde{\bfB}  \right\|_{F}^2 + \alpha \sum_{i=1}^{k}\|\bfM_i\|_2, \nonumber \\
    \st ~& \widetilde{\bfM} = \bfZ,
\end{align}
where
\begin{align}
\widetilde{\bfM}=\begin{bmatrix}\bfM_1&\cdots&\bfM_k\end{bmatrix}, \hspace{1em}
\widetilde{\bfB}=\begin{bmatrix}\bfB_1 \\ \vdots \\ \bfB_k\end{bmatrix}.
\end{align}

The augmented Lagrangian of \refEq{eq:admm-noisy} is
\begin{align}
    \mathcal{L}_{\mu}\left(\widetilde{\bfM},\bfZ,\bfY\right) &= \half \left\| \bfW - \bfZ\widetilde{\bfB}  \right\|_{F}^2 + \alpha \sum_{i=1}^{k}\|\bfM_i\|_2 \nonumber \\
    & + \left<\bfY,\widetilde{\bfM}-\bfZ\right> + \frac{\mu}{2}\left\| \widetilde{\bfM} - \bfZ \right\|_F^2,
\end{align}
where $\bfY$ is the dual variable and $\mu$ is a parameter controlling the step size in optimization. Then, the ADMM alternates the following steps until convergence:
\begin{align}
    &\widetilde{\bfM}^{t+1} = \arg\min_{\widetilde{\bfM}}\mathcal{L}_{\mu}\left(\widetilde{\bfM},\bfZ^t,\bfY^t\right); \label{eq:admm1} \\
    &\bfZ^{t+1} = \arg\min_{\bfZ}\mathcal{L}_{\mu}\left(\widetilde{\bfM}^{t+1},\bfZ,\bfY^t\right); \label{eq:admm2} \\
    &\bfY^{t+1} = \bfY^{t} + \mu~\left(\widetilde{\bfM}^{t+1}-\bfZ^{t+1}\right). \label{eq:admm3}
\end{align}

The subproblem in \refEq{eq:admm1} can be rewritten as
\begin{align}\label{eq:admm11}
&\min_{\widetilde{\bfM}}\mathcal{L}_{\mu}\left(\widetilde{\bfM},\bfZ^t,\bfY^t\right) \nonumber \\
=&\min_{\widetilde{\bfM}} \frac{1}{2}\left\| \widetilde{\bfM} - \bfZ^t + \frac{1}{\mu}\bfY^t \right\|_F^2 + \frac{\alpha}{\mu} \sum_{i=1}^{k}\|\bfM_i\|_2 \nonumber \\
=&\min_{\bfM_1,\cdots,\bfM_k} \sum_{i=1}^{k} \left\{ \frac{1}{2}\left\| \bfM_i - \bfQ_i^t \right\|_F^2 + \frac{\alpha}{\mu}\|\bfM_i\|_2 \right\},
\end{align}
where $\bfQ_i^t$ is the $i$-th column-triplet of $\bfZ^t - \frac{1}{\mu}\bfY^t$. Therefore, each $\bfM_i$ can be updated separately by solving a proximal problem based on \refProp{prop:prox2norm}:
\begin{align}\label{eq:updateM}
\bfM_i^{t+1} = \mathcal{D}_{\frac{\alpha}{\mu}}(\bfQ_i^t), ~~\forall i\in[1,k].
\end{align}

For the subproblem in \refEq{eq:admm2}, $\mathcal{L}_{\mu}\left(\widetilde{\bfM}^{t+1},\bfZ,\bfY^t\right)$ is a quadratic form and admits a closed-form solution:
{\small
\begin{align}\label{eq:updateZ}
\bfZ^{t+1} = \left( \bfW\widetilde{\bfB}^T+\mu\widetilde{\bfM}^{t+1}+\bfY^t \right) \left( \widetilde{\bfB}\widetilde{\bfB}^T+\mu\bfI \right)^{-1}.
\end{align}
}

The steps are summarized in \refAlg{alg:admm-noisy}. Following the standard proof in \cite{boyd2010distributed}, it can be shown that the sequences of values produced by the iterations in \refAlg{alg:admm-noisy} converge to the optimal solution of the problem \refEq{eq:admm-noisy}, which is also the optimal solution to the original problem \refEq{eq:finalnoisy}. In implementation, the convergence criterion and the adaptive scheme for tuning step size $\mu$ from \cite[Section 3]{boyd2010distributed} are adopted.

\begin{algorithm}
\LinesNumbered
\caption{ADMM to solve \refEq{eq:finalnoisy}}\label{alg:admm-noisy}
\small
\KwIn{$\bfW$, $\alpha$}
\KwOut{$\bfM_1,\cdots,\bfM_k$}
\vspace{0.5em}
initialize $\bfZ=\bfY=\bfzero$, $\mu>0$\;
\While{not converged}{
parallel \For{$i = 1$ \KwTo $k$}{
$\bfQ_i^t=$ the $i$-th column-triplet of $\bfZ^t - \frac{1}{\mu}\bfY^t$ \;
$\bfM_i^{t+1} = \mathcal{D}_{\frac{\alpha}{\mu}}(\bfQ_i^t)$ \;
}
$\bfZ^{t+1} = \left( \bfW\widetilde{\bfB}^T+\mu\widetilde{\bfM}^{t+1}+\bfY^t \right) \left( \widetilde{\bfB}\widetilde{\bfB}^T+\mu\bfI \right)^{-1}$ \;
$\bfY^{t+1} = \bfY^{t} + \mu~\left(\widetilde{\bfM}^{t+1}-\bfZ^{t+1}\right)$ \;
}
\vspace{0.5em}
\end{algorithm}
\vspace{-1em}

\subsection{Reconstruction}\label{sec:reconstr}

After \refEq{eq:finalnoisy} is solved, two algorithms for reconstructing the 3D shape from the estimated $\bfM_i$s are proposed.

\subsubsection{Direct reconstruction}\label{sec:direct}

The 3D shape is reconstructed according to the relaxed shape model \refEq{eq:shape-model-relaxed}, where $c_i$ and $\bfR_i$ are recovered from the estimated $\bfM_i$. The steps are summarized in \refAlg{alg:direct}, where $\bfm_i^{(j)}$ and $\bfr_i^{(j)}$ denote the $j$-th row vectors of $\bfM_i$ and $\bfR_i$, respectively. Note that $c_i=-\|\bfM_i\|_2$ is also a feasible solution. To eliminate the ambiguity, we assume that $c_i\geq 0$ and impose this constraint when learning the shape dictionary.

\begin{algorithm}\label{alg:direct}
\LinesNumbered
\caption{Direct reconstruction}
\KwIn{$\bfM_1,\cdots,\bfM_k$}
\KwOut{$\bfS$}
\vspace{0.5em}
\For{$i = 1$ \KwTo $k$}{
$c_i=\|\bfM_i\|_2$\;
$\bfr_i^{(1)}=\bfm_i^{(1)}/c_i$\;
$\bfr_i^{(2)}=\bfm_i^{(2)}/c_i$\;
$\bfr_i^{(3)}=\bfr_i^{(1)}\times\bfr_i^{(2)}$\;
$\bfR_i=[\bfr_i^{(1)},\bfr_i^{(2)},\bfr_i^{(3)}]^T$\;
}
$\bfS = \sum_{i=1}^{k} c_i\bfR_i\bfB_i$\;
\vspace{0.5em}
\end{algorithm}
\vspace{-1em}

\subsubsection{Rotation synchronization and refinement}\label{sec:refine}

The 3D shape is reconstructed according to the original shape model \refEq{eq:shape-model-original}. In order to recover the coefficient vector $\bfc$ and a single rotation $\bar\bfR$ from the estimated $\bfM_i$s, the following rotation synchronization problem is solved:
\begin{align}\label{eq:rot_avg}
    \min_{\bfc,\bar\bfR} & \sum_{i=1}^{k}\| \bfM_i - c_i \bar\bfR\|_F^2, \nonumber \\
    \st ~& \bar\bfR\bar\bfR^T = \bfI_2.
\end{align}
The problem in \refEq{eq:rot_avg} corresponds to the ``metric projection" step in \cite{paladini2012optimal,del2012bilinear}, which can be exactly solved by semidefinite programming \cite{paladini2012optimal} or approximately by a more efficient algorithm \cite[Section 4.4.1]{del2012bilinear}.

The solution from the synchronization described above is not necessarily the solution to the original problem \refEq{eq:originalnoisy}, but it can be used as a good initialization and refined by the alternating minimization in \refAlg{alg:altern}. In \refSec{sec:experiments}, it will be demonstrated that such a better initialization can clearly improve the performance of alternating minimization.
A potential issue in the synchronization is that, when the rotations obtained from the convex program are very different from each other, the consensus of them might be meaningless and fails to provide a good initialization.

The full steps are summarized in \refAlg{alg:refine}.

\begin{algorithm}\label{alg:refine}
\LinesNumbered
\caption{Refinement and reconstruction}
\KwIn{$\bfM_1,\cdots,\bfM_k$}
\KwOut{$\bfS$}
\vspace{0.5em}
initialize $\bfc$ and $\bar\bfR$ by solving \refEq{eq:rot_avg} using the algorithm in \cite[Section 4.4.1]{del2012bilinear}\;
\vspace{0.25em}
refine $\bfc$ and $\bar\bfR$ by \refAlg{alg:altern}\;
\vspace{0.25em}
$\bfS = \sum_{i=1}^{k} c_i\bfB_i$\;
\vspace{0.5em}
\end{algorithm}
\vspace{-1em}

\subsection{Outlier modeling}\label{sec:outliers}

In real applications, the 2D correspondences are given by detectors and there are likely to be gross errors (outliers). In this case, we explicitly model the outliers by a sparse matrix $\bfE$ and modify the formulation in \refEq{eq:finalnoisy} as
\begin{align}\label{eq:finalrobust}
    \min_{\bfM_1,\cdots,\bfM_k,\bfE,\bfT}~ & \half \left\| \bfW - \sum_{i=1}^{k} \bfM_i\bfB_i -\bfE - \bfT\bfone^T \right\|_{F}^2 \nonumber \\
    & + \alpha \sum_{i=1}^{k}\|\bfM_i\|_2 + \beta \|\bfE\|_1.
\end{align}
Note that in this case the translation $\bfT$ cannot be eliminated by centralizing the data at the beginning.

The problem in \refEq{eq:finalrobust} can be solved by ADMM as well. The algorithm is presented in \refApp{sec:alg-robust}.

\subsection{Dictionary learning}\label{sec:dl}

In general, one can simply use all existing or a few representative 3D models as the basis shapes. But when the collection is very large, e.g., a motion capture dataset often contains thousands of human skeletons, it is necessary to use the dictionary learning technique to learn a dictionary of basis shapes that can concisely summarize the variability in training data and enable a sparse representation. The dictionary learning problem can be formulated as follows:
\begin{align}\label{eq:dl}
    \min_{\bfB_1,\cdots,\bfB_k,\bfC}~ & \sum_{j=1}^{n} \half \| \bfS_j - \sum_{i=1}^{k} C_{ij}\bfB_i \|_{F}^2 + \lambda \|\bfC\|_1 \nonumber \\
    \st & ~C_{ij} \geq 0, ~ \|\bfB_i\|_F \leq 1, \nonumber \\
    &~ \forall i\in[1,k], ~ j\in[1,n],
\end{align}
where $\bfS_j$s denote the training shapes, $\bfB_i$s are the basis shapes to be learned, and $C_{ij}$ represents the $i$-th coefficient for the $j$-th training shape. The two terms in the cost function correspond to the reconstruction error and the sparsity of representation, respectively. The nonnegativity constraint is introduced to remove the ambiguity as discussed in \refSec{sec:direct}. The problem in \refEq{eq:dl} is locally solved by alternately updating $\bfC$ and $\bfB_i$s via projected gradient descent, an algorithm widely used in dictionary learning that converges to a local optimum \cite{mairal2010online}. The algorithm is summarized in \refApp{sec:alg-dl}.

\section{Experiments}\label{sec:experiments}

\subsection{Exact recovery}\label{sec:simulation}

\begin{figure}
\centering
\includegraphics[width=0.51\linewidth]{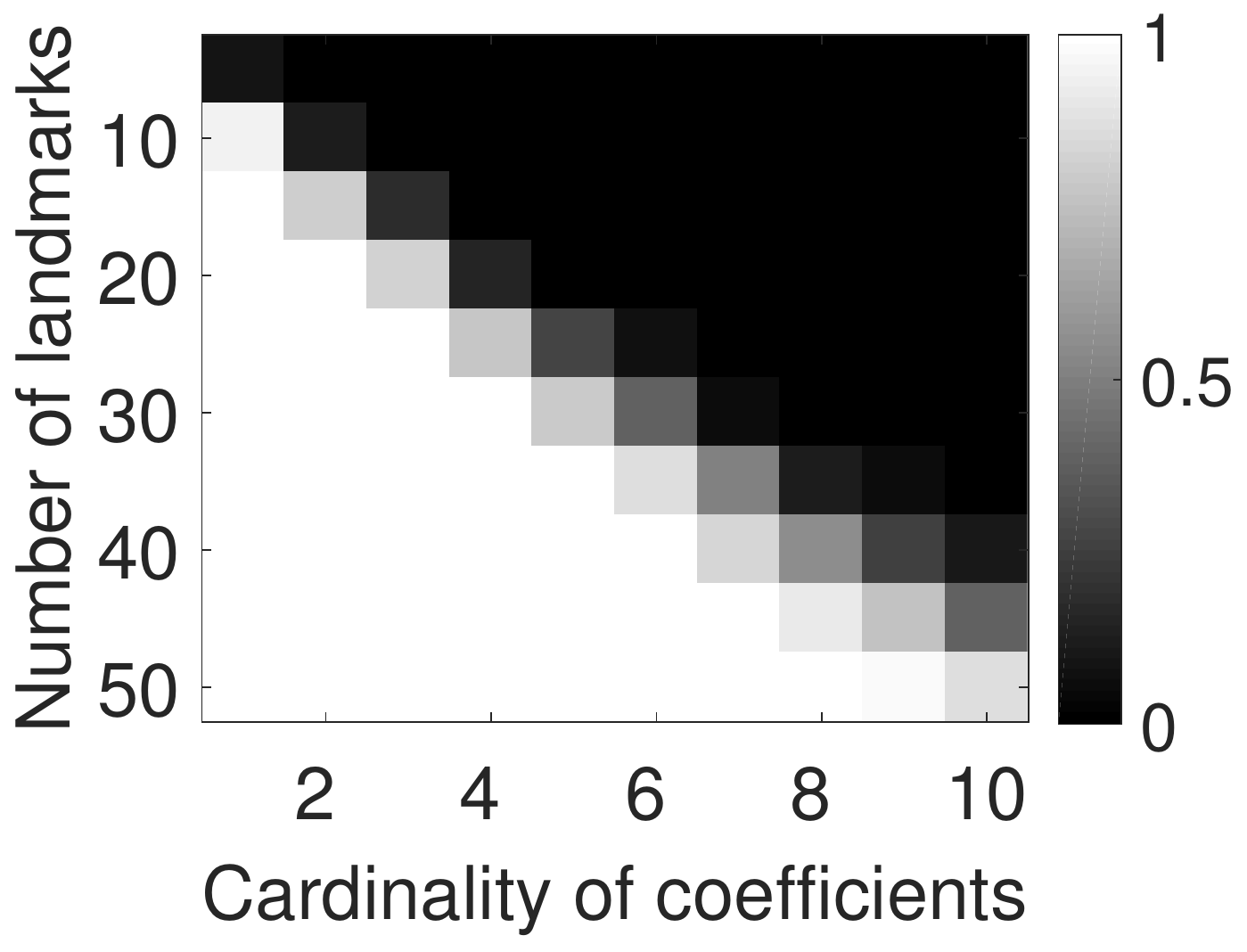}
\caption{The frequency of exact recovery. The intensity indicates the frequency of success under a variety of problem settings.} \label{fig:freqexact}
\vspace{-1em}
\end{figure}

We investigate whether the spectral-norm minimization in \refEq{eq:finalnoiseless} can exactly solve the ill-posed inverse problem based on the prior knowledge of sparsity and orthogonality using simulation.

More specifically, the data is synthesized with the following model:
\begin{align}\label{eq:simu-model}
\bfW =\sum_{i=1}^{k}\bfM_i\bfB_i,
\end{align}
where $\bfB_1,\cdots,\bfB_k\in\RR{3}{p}$ ($k=50$ and $p$ is varying) are randomly generated bases with entries sampled independently from the normal distribution $\mathcal{N}(0,1)$. Each $\bfM_i\in\RR{2}{3}$ is generated from $\bfM_i=c_i\bar{\bfR}_i$, where $c_i$ is sampled from a uniform distribution $\mathcal{U}(0,1)$ and $\bar{\bfR}_i$ denotes the first two rows of a rotation matrix with angles uniformly sampled from $[0,2\pi]$. Only a randomly selected subset of $c_1,\cdots,c_k$ are left as nonzero with a varying cardinality of $z$. Then, $\bfW$ and $\bfB$ are taken as input and $\bfM_i$s are estimated by solving \refEq{eq:finalnoiseless}. The recovery error is defined as
\begin{align}
\mbox{relative error} = {\|\hat\bfM-\widetilde{\bfM}\|_F}/{\|\widetilde{\bfM}\|_F},
\end{align}
where $\bfM_i$s are concatenated in $\widetilde{\bfM}$, and $\hat\bfM$ is the algorithm estimate. A recovery is regarded as exact if the relative error $< 10^{-3}$.

\refFig{fig:freqexact} shows the frequency of exact recovery with varying $p$ (number of landmarks) and $z$ (cardinality of true coefficients), which is evaluated over 100 trials for each setting. Note that the number of unknowns ($6k$) is much larger than the number of equations ($2p$). The proposed convex program can exactly solve the problem with a frequency equal to 1 in the lower-triangular area, where the number of landmarks is sufficiently large and the coefficients are truly sparse. This demonstrates the power of convex relaxation, which has proven to be successful in many inverse problems, e.g., compressed sensing \cite{candes2008introduction} and matrix completion \cite{candes2010power}. The performance drops in more ill-posed cases in the upper-triangular area. This observation is analogous to the phase transition in compressive sensing, where the recovery probability also depends on the number of observations and the underlying signal sparsity \cite{donoho2009observed}.

Note that the original model \refEq{eq:bilinear} is a special case of the relaxed model \refEq{eq:simu-model} when $\bfM_i$s are the scaled versions of each other. The result obtained by using the original model for data generation is similar to \refFig{fig:freqexact}.

\subsection{Human pose estimation}\label{sec:human}

The applicability of sparse representation for 3D human pose recovery has been thoroughly studied in previous work \cite{ramakrishna2012reconstructing,wang2014robust,fan2014pose}. In this paper, we aim to illustrate the advantage of the proposed convex program compared to the alternating minimization scheme.

Empirical evaluation is carried out on the CMU motion capture dataset \cite{mocap}. The dataset includes thousands of sequences of 3D human skeletons and part of them are selected for evaluation. More specifically, eight motion categories are considered including walking, running, jumping, climbing, boxing, dancing, sitting and basketball. For each motion, six sequences are collected from different subjects with three sequences for training and the rest for testing. Then, a dictionary is learned with all training sequences from various motions and used to reconstruct the test sequences. A 15-joint model (i.e., head, thorax, pelvis, shoulders, elbows, wrists, hips, knees, and ankles) is adopted to represent a human skeleton.

\begin{figure}
\centering
\includegraphics[width=0.5\linewidth]{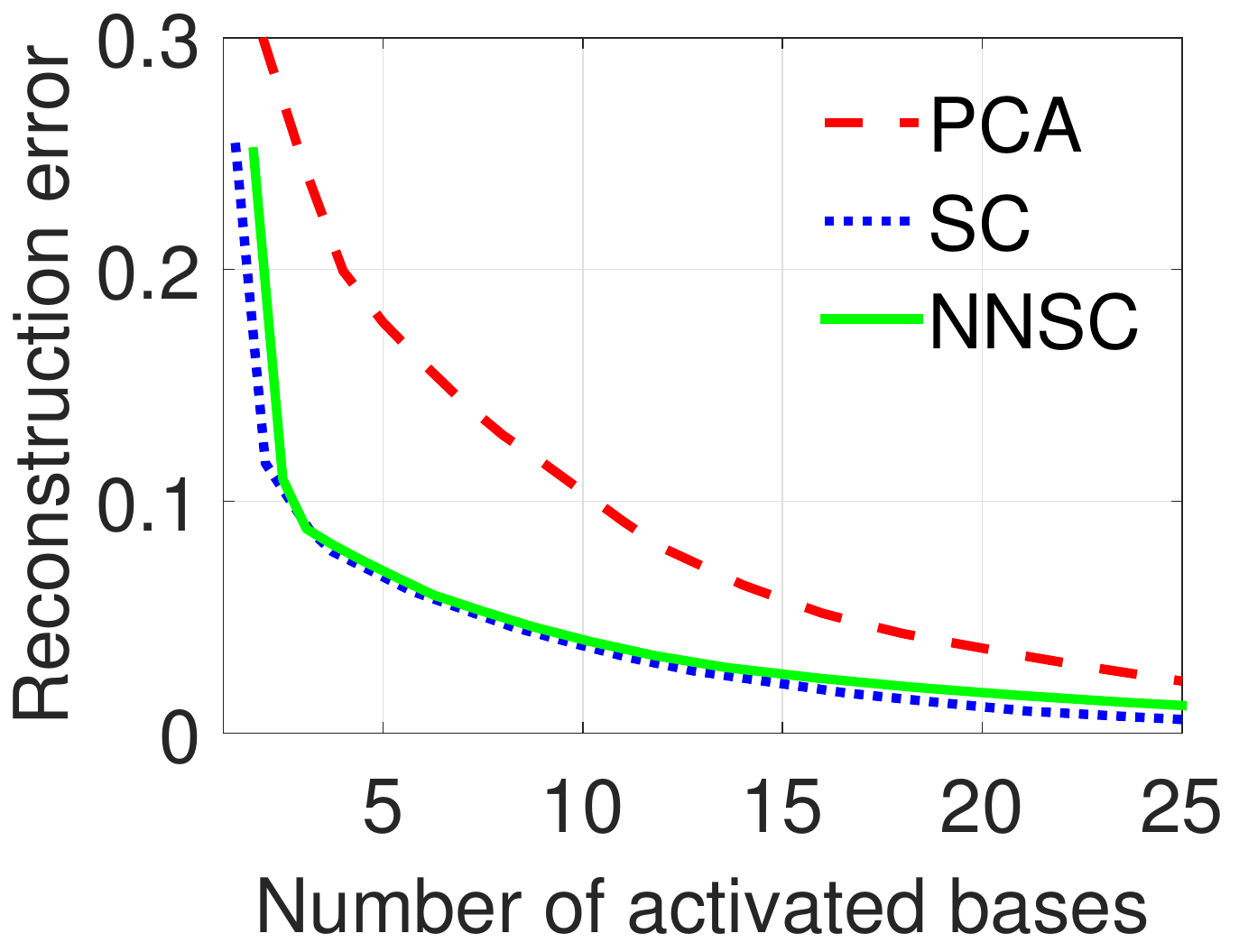}
\caption{{Representability of principal component analysis (PCA), sparse coding (SC) and nonnegative sparse coding (NNSC) for human poses.}} \label{fig:test-dl}
\end{figure}

\begin{figure*}
  \centering
  \begin{minipage}[b]{0.52\linewidth}
  \centering
  \includegraphics[width=\linewidth]{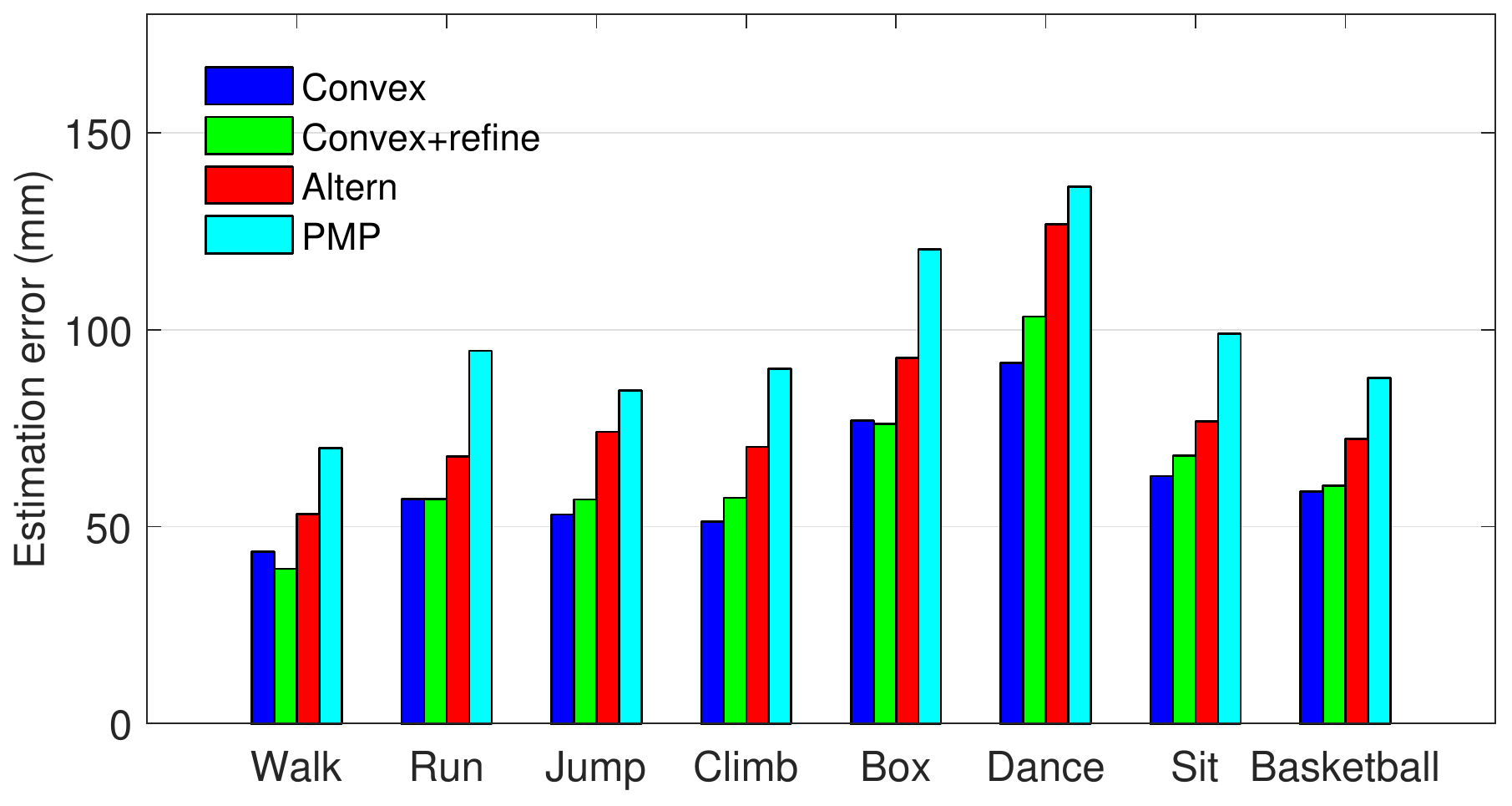} \\
  \centerline{(a)}
  \end{minipage}
  \hspace{1em}
  \begin{minipage}[b]{0.4\linewidth}
  \centering
  \includegraphics[width=\linewidth]{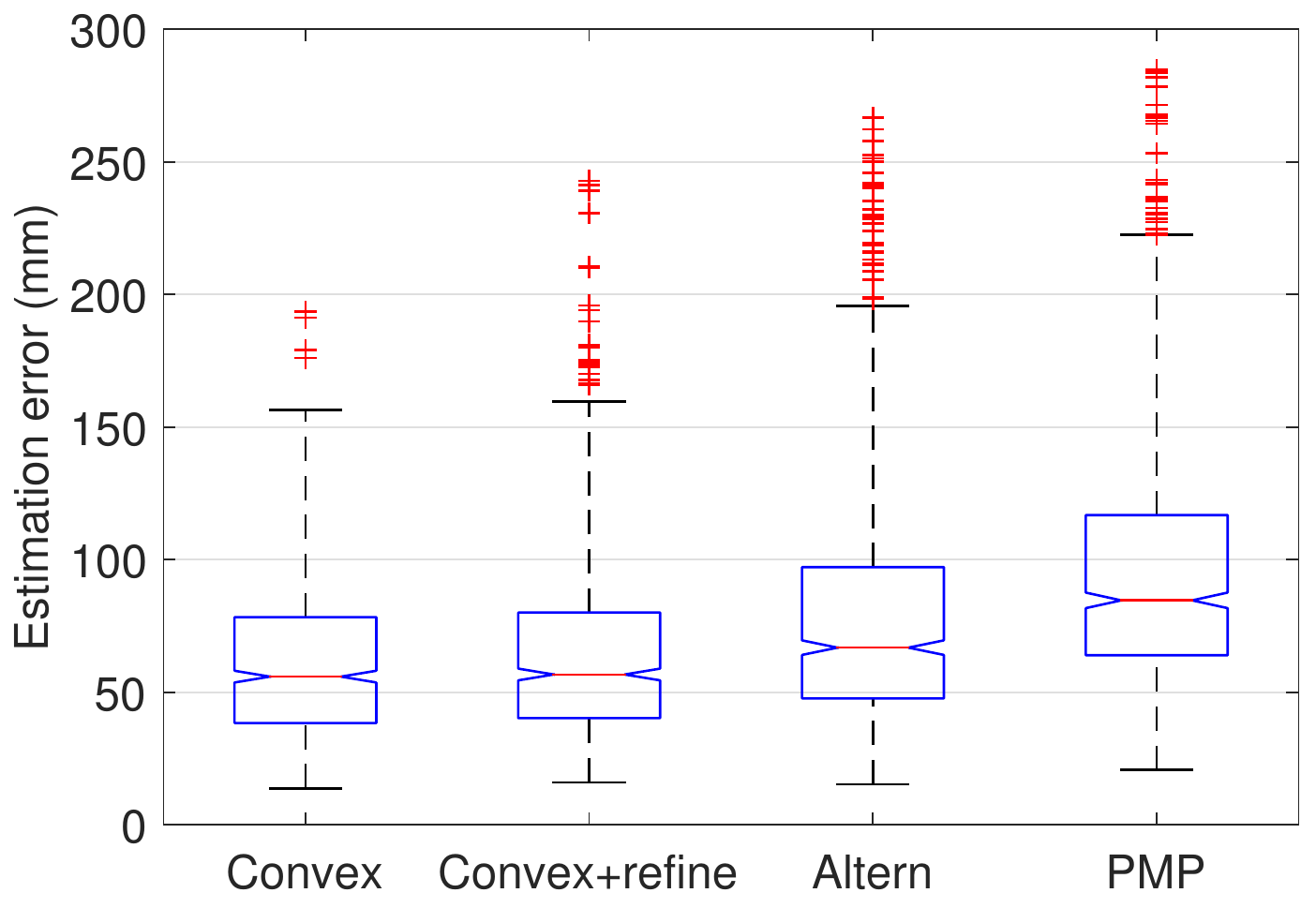} \\
  \centerline{(b)}
  \end{minipage}\\
  \vspace{1em}
  \begin{minipage}[b]{0.25\linewidth}
  \centering
  \includegraphics[width=\linewidth]{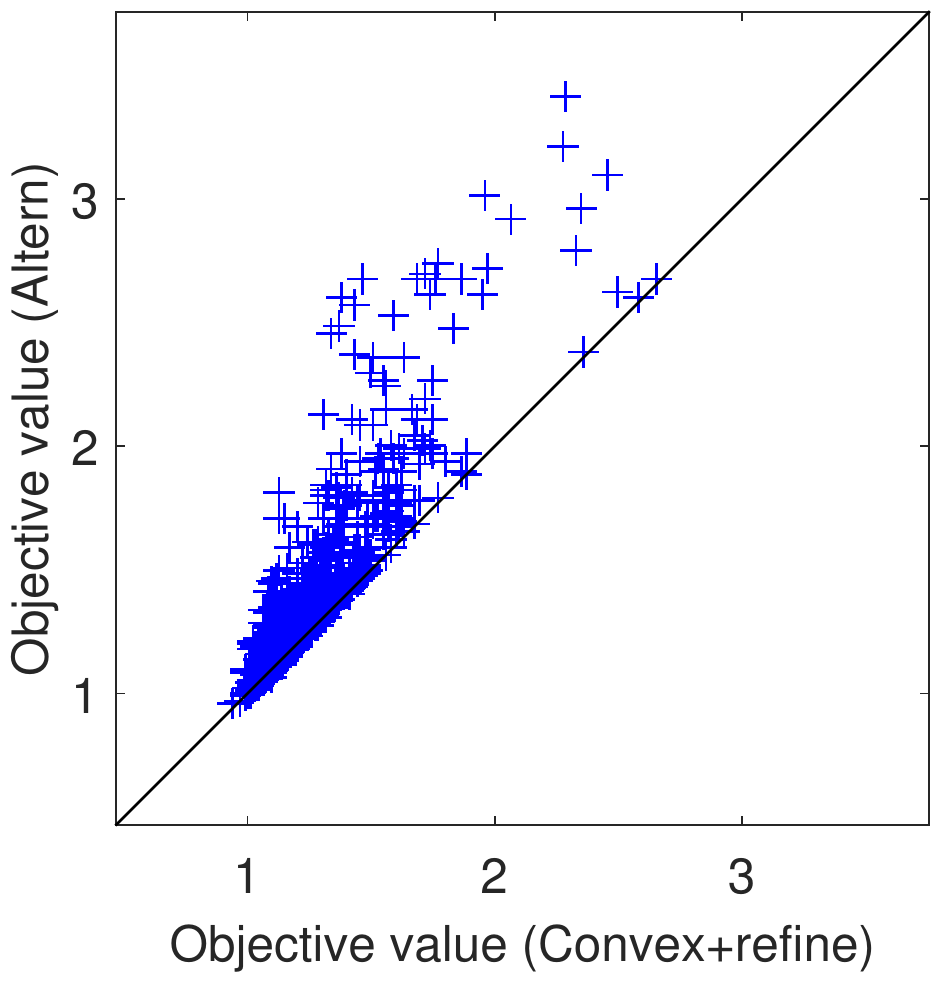} \\
  \centerline{(c)}
  \end{minipage}
  \hspace{1em}
  \begin{minipage}[b]{0.33\linewidth}
  \centering
  \includegraphics[width=\linewidth]{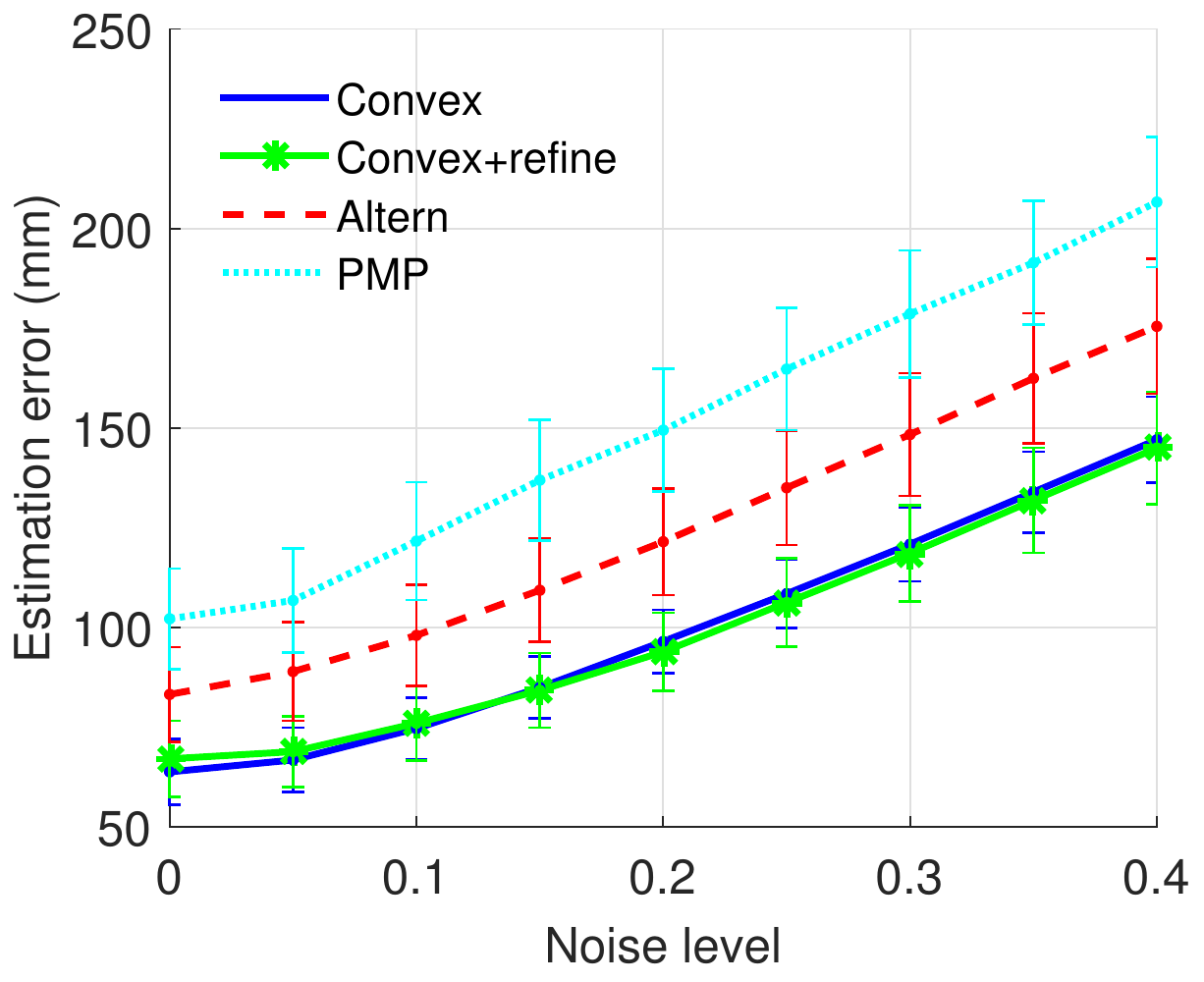} \\
  \centerline{(d)}
  \end{minipage}
  \begin{minipage}[b]{0.33\linewidth}
  \centering
  \includegraphics[width=\linewidth]{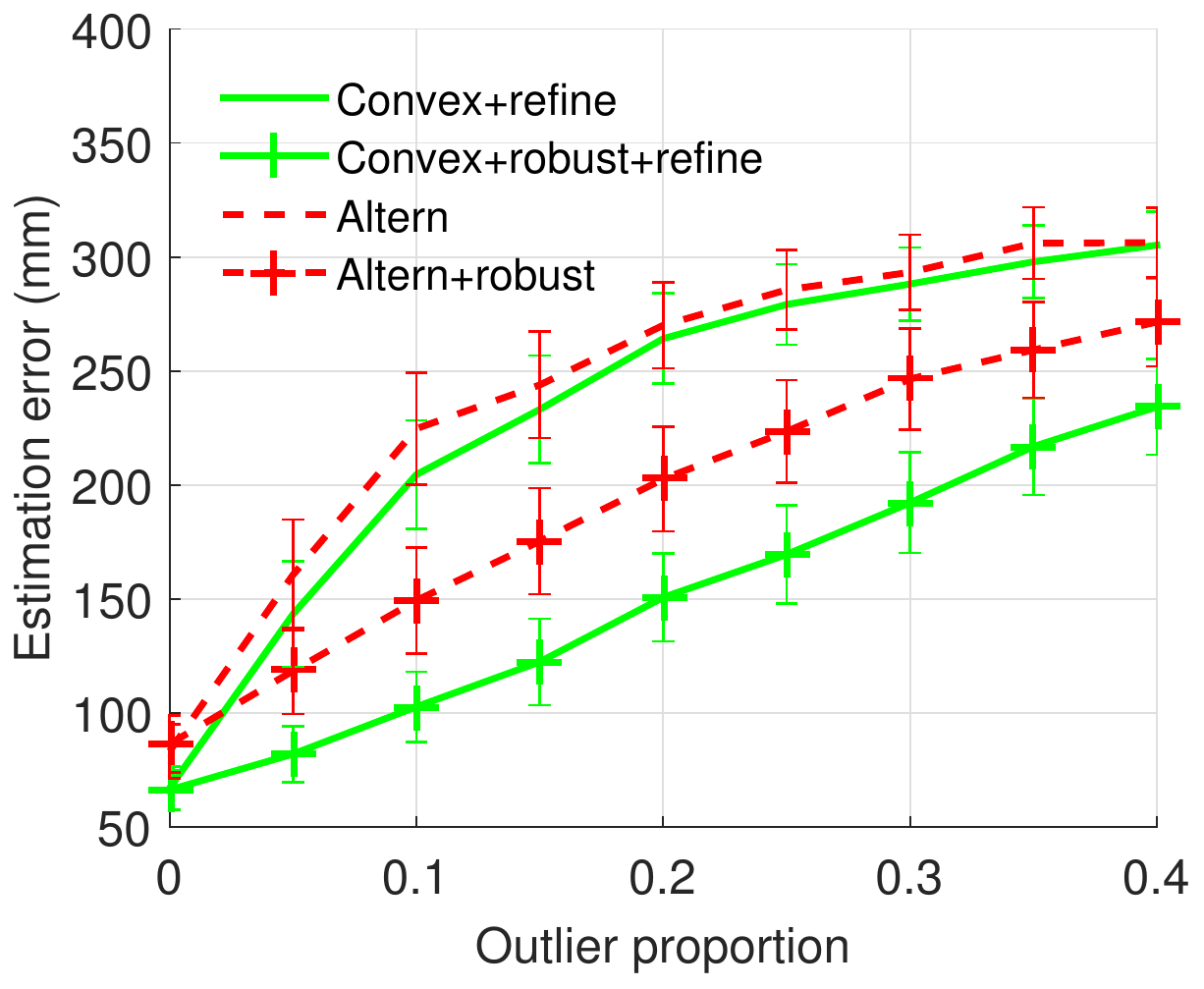} \\
  \centerline{(e)}
  \end{minipage}
  \caption{Quantitative results on the CMU motion capture dataset. (a) The mean 3D estimation errors for different motions. (b) The box plot of estimation errors. (c) The comparison between objective values achieved by the mean shape initialization and by the convex initialization. (d) The sensitivity to Gaussian noise. (e) The sensitivity to outliers. The length of error bar in (d) and (e) indicates half standard deviation. }\label{fig:mocap-quant}
\end{figure*}

To build the shape dictionary, all 3D poses in the training set are collected and aligned by the Procrustes method. The size of the dictionary $k$ is set as 128 yielding an overcomplete dictionary since the dimension of the basis vector is 45 (15 joints in 3D). The dictionary is initialized by uniformly sampling $k$ poses from the training data and updated by the algorithm described in \refSec{sec:dl}. The representability of the learned dictionary is measured by the reconstruction error, i.e., the first term in \refEq{eq:dl}. To evaluate the representability, \refEq{eq:dl} is solved multiple times with a variety of $\lambda$ yielding a sequence of reconstruction errors and the corresponding sparsity levels. The more general case without the nonnegativity constraint is also tested. The result is shown in \refFig{fig:test-dl}, which indicates that the sparse coding needs much fewer activated bases compared to PCA to achieve the same reconstruction error, {and adding the nonnegativity constraint sacrifices the representability but the difference is small in our experiment}.

To generate 2D testing data, an orthographic camera rotating around the subject (360 degrees per sequence) is simulated and the 3D joints are projected to 2D with the simulated camera. The following methods are compared:
\begin{itemize}
\item {\bf Convex}: the proposed convex program with direct reconstruction (\refAlg{alg:admm-noisy}+\refAlg{alg:direct}).
\item {\bf Convex+refine}: the proposed convex program followed by refinement (\refAlg{alg:admm-noisy}+\refAlg{alg:refine}). The local update of rotation is implemented by the manifold optimization over the Stiefel manifold with the Manopt toolbox \cite{boumal2014manopt}.
\item {\bf Altern}: the alternating algorithm with the mean shape as initialization (\refAlg{alg:altern}). The rotation is updated by SVD as in previous work \cite{ramakrishna2012reconstructing,fan2014pose,zhou2014sptio}.
\item {\bf PMP}: Projected Matching Pursuit from Ramakrishna et al. \cite{ramakrishna2012reconstructing}, which solves sparse coding by greedily selecting basis shapes instead of $\ell_1$ minimization. The dictionary for PMP is constructed by the method provided in the original paper but with the same training data as ours.
\end{itemize}

\begin{figure*}
\centering
\includegraphics[width=0.9\linewidth]{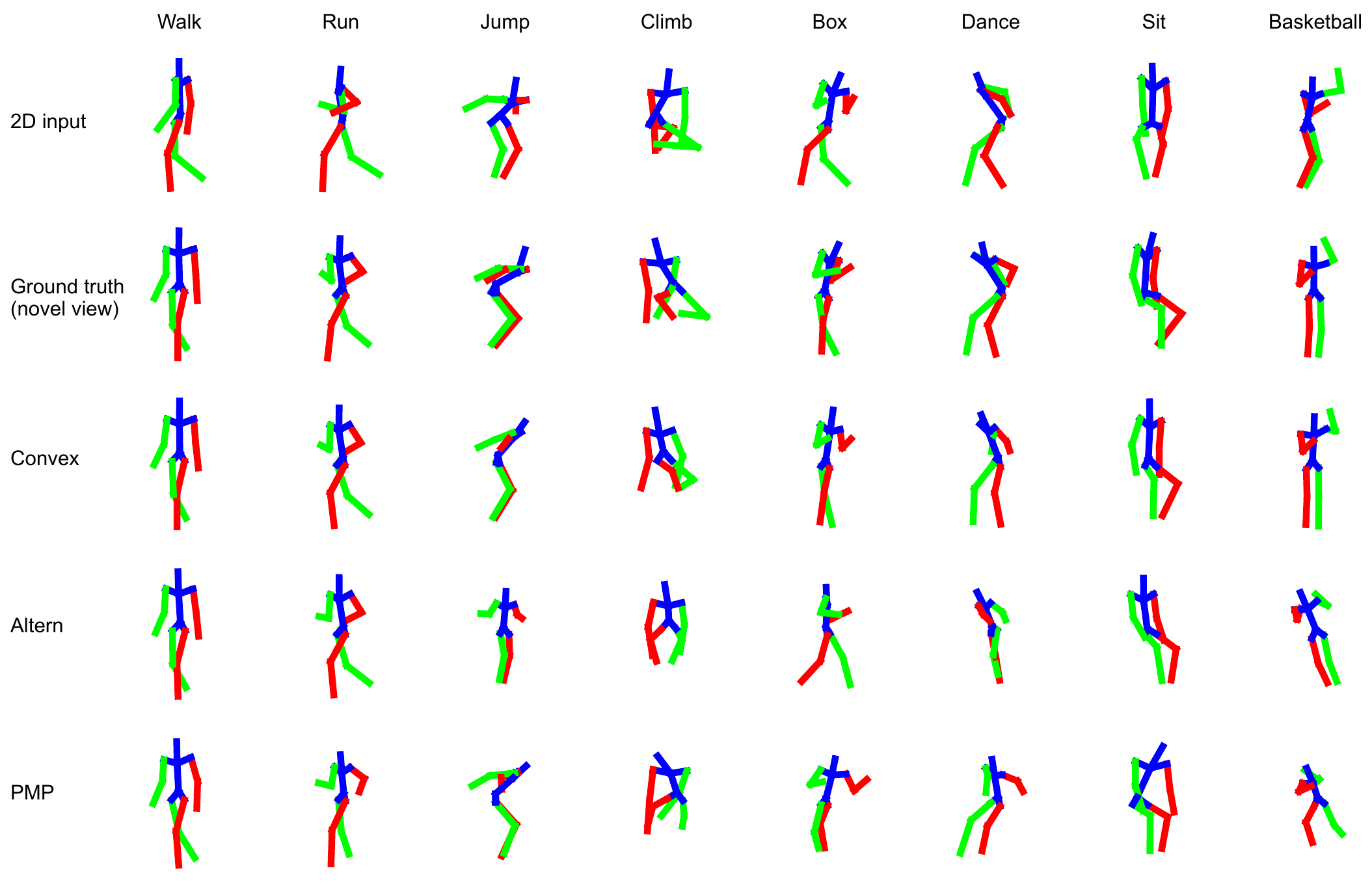}
  \caption{Qualitative results on the CMU motion capture dataset. The rows from top to bottom correspond to the input 2D poses, the ground-truth 3D poses (visualized in a novel view), and the reconstructions from the proposed method, the alternating minimization, and the PMP method \cite{ramakrishna2012reconstructing}, respectively. Red and green indicate left and right, respectively. }
  \label{fig:mocap-qual}
\end{figure*}

\begin{figure*}
  \centering
  \includegraphics[width=0.8\linewidth]{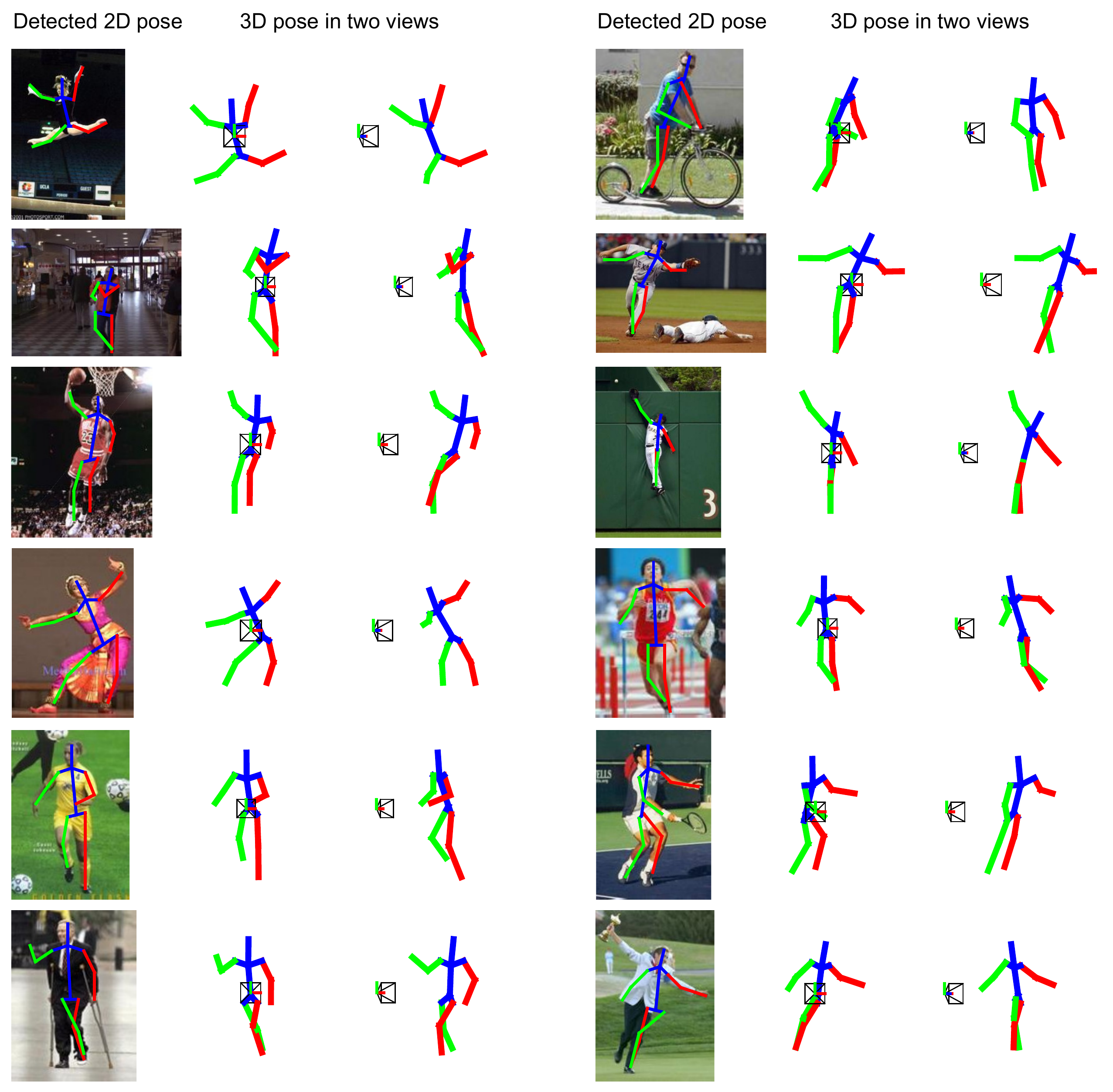}
  \caption{Qualitative results on the PARSE dataset with 2D poses detected by an existing 2D pose detector \cite{yang2011articulated}. In each example, the detected 2D pose superposed on the original image and the reconstruction in two different views are shown. The last row shows two failed examples.}\label{fig:parse}
\end{figure*}

The estimation error is evaluated by the 3D Euclidian distances between the estimated joint locations and the ground truth in the camera frame up to a translation and scaling. The mean errors for various motions are shown in \refFig{fig:mocap-quant}(a) and the statistics of errors in \refFig{fig:mocap-quant}(b). The proposed convex method consistently outperforms the nonconvex methods for all motions. Comparing ``convex+refine" and ``altern" which solve the same problem \refEq{eq:originalnoisy} at the end, the error is apparently decreased by using better initialization (the convex solution). To further verify the fact that the alternating minimization depends on initialization, we compare the objective values achieved by the alternating algorithm initialized by the convex solution and the mean shape. As shown in \refFig{fig:mocap-quant}(c), the objective values achieved by convex initialization are always smaller. {The mean errors of ``convex" and ``convex+refine" are comparable on the average. The rotation synchronization and refinement may lead to worse reconstructions if the convex solution fails to provide a good initialization or the optimal solution to \refEq{eq:originalnoisy} is not necessarily the best reconstruction, which may happen when the pose to be reconstructed cannot be well represented by the learned dictionary.}


Some visual examples of reconstructed poses for various motions are shown in \refFig{fig:mocap-qual}. All methods perform well in the ``walk" and ``run" examples, where the poses are close to the mean pose (stand upright). But for more complex motions, the nonconvex algorithms may be stuck at local optima as shown in other examples, while the proposed convex algorithm still obtains appealing results. A demonstration video showing several reconstructed 3D pose sequences is included in the supplementary material. The reconstructions by the convex method are much smoother over time compared to the ones by the nonconvex method which suffer from abrupt changes due to the inherent instability of nonconvex optimization.

The robustness of the proposed method against Gaussian noise is also tested. Gaussian noise with a varying standard deviation is added to the 2D input. The estimation error versus the standard deviation of the simulated noise is shown in \refFig{fig:mocap-quant}(d). The performances of the convex methods are consistently better than the alternating ones under all noise levels. To test the robustness of the extended model in \refEq{eq:finalrobust} against outliers, a proportion of landmarks are selected with their 2D locations changed to be some random values within a proper range (roughly the rectangle that covers the skeleton). For comparison, the nonconvex model \refEq{eq:originalnoisy} is also extended by introducing an outlier term similar to the one in \refEq{eq:finalrobust} and minimized by the alternating scheme. The curves of estimation error versus outlier proportion are shown in \refFig{fig:mocap-quant}(e). The robust models clearly outperform the original models and the convex method achieves a better performance than the alternating one.

To illustrate the real applicability of the proposed method, we apply the 2D pose detector from Yang and Ramanan \cite{yang2011articulated} on the PARSE dataset \cite{ramanan2006learning}, and use the proposed method to lift the detected 2D poses to 3D with the dictionary learned on the CMU dataset. Some selected results are shown in \refFig{fig:parse}.

\subsection{Car model estimation}

\begin{figure*}
  \centering
  \includegraphics[width=0.9\linewidth]{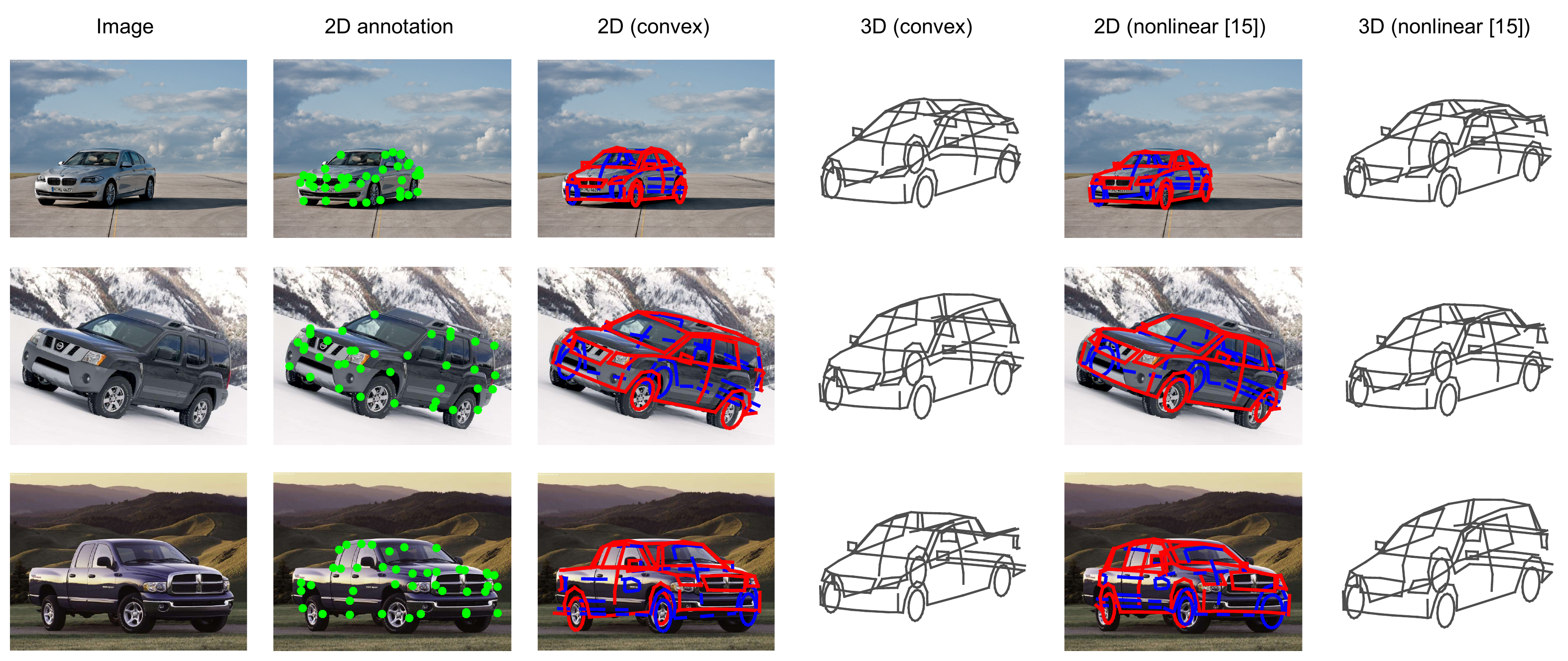}
  \caption{Qualitative results on the FG3DCar dataset given 2D correspondences. The columns correspond to the original image, the input 2D landmarks, and the 2D/3D models output by the proposed method and the nonlinear optimization \cite{lin2014jointly}, respectively. Only visible landmarks ($\sim40$ per image) are used for model fitting. The 3D models are visualized in a novel view different from the original image. The car models are the BMW 5 Series 2011 (sedan), the Nissan Xterra 2005 (SUV) and the Dodge Ram 2003 (pick-up truck), respectively.} \label{fig:fg3dcar-qual}
\end{figure*}

The applicability of the proposed method for car model estimation is demonstrated using the Fine-Grained 3D Car (FG3DCar) dataset \cite{lin2014jointly}, which provides images of cars, 2D landmark annotations and some 3D models. The 3D models of 15 cars are concatenated as the shape dictionary to reconstruct the other cars from the visible landmarks annotated in the images ($\sim$40 points per image). Several examples of reconstruction are shown in \refFig{fig:fg3dcar-qual}. For comparison, the reconstructions from the algorithm proposed in the original paper \cite{lin2014jointly} are also shown, which adopts a perspective camera model and locally updates the camera and shape parameters with a Jacobian system initialized by the mean shape and predefined camera parameters. The nonlinear optimization performs well in the sedan example but fails in the SUV and truck examples, where the models deviate far away from the mean shape. Similar results were reported in the original paper \cite{lin2014jointly} and the authors proposed to use the class-specific mean shape for better initialization. Instead, the proposed method can achieve appealing results with arbitrary initialization.

We observed that the alternating minimization in \refAlg{alg:altern} also performed very well for car model reconstruction, as the shape variability of cars was relatively small and the mean-shape initialization was very close to the true solution. But when there are outliers in the observation, the initial estimate of viewpoint might be unreliable and the alternating algorithm may fail. To validate this, some outliers are simulated. More specifically, for each image 20 annotated landmarks are randomly selected with their locations changed to be random values in the image plane. Then, the robust models are solved by the convex method and by the alternating algorithm initialized with the mean shape, respectively. Since there is no 3D ground truth (the 3D models provided in the dataset were reconstructed by the authors), the 2D shape fitting error is evaluated, i.e., comparing the 2D projection of the reconstructed 3D model with the original 2D annotations, to see whether the synthesized outliers can be corrected. The comparison is shown in \refFig{fig:fg3dcar-quant}. While the estimation errors are identical for most examples, there are still many examples in which the convex method performs apparently better than the alternating method. The same conclusion can be drawn by comparing the objective values. {The mean 2D errors of the ``altern'', ``convex'' and ``convex+refine'' methods are 44.47, 37.21 and 34.76 pixels, respectively. The rotation synchronization and refinement is beneficial here as the shape variability is small and a more constrained model is preferred.}

\begin{figure}
\begin{minipage}[b]{0.48\linewidth}
  \centering
  \includegraphics[width=\linewidth]{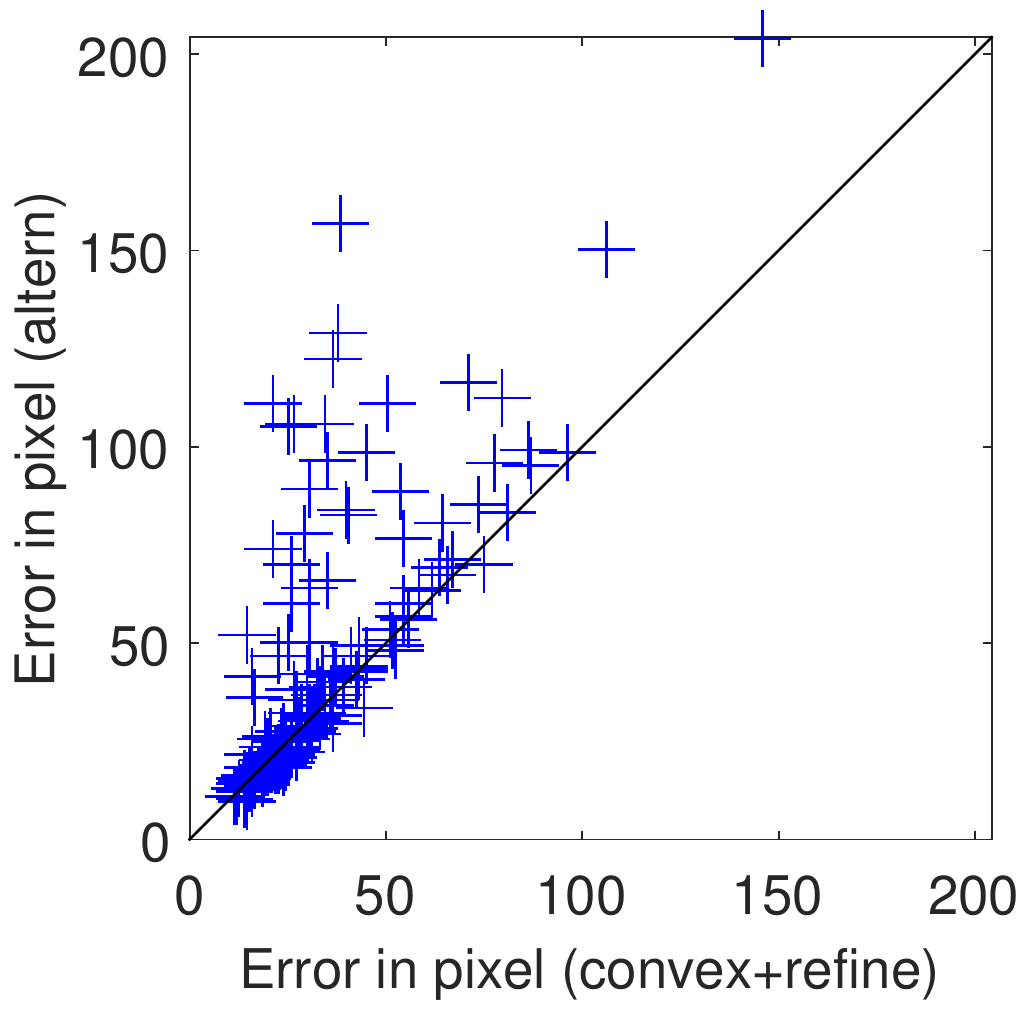} \\ (a)
\end{minipage}
\begin{minipage}[b]{0.48\linewidth}
  \centering
  \includegraphics[width=\linewidth]{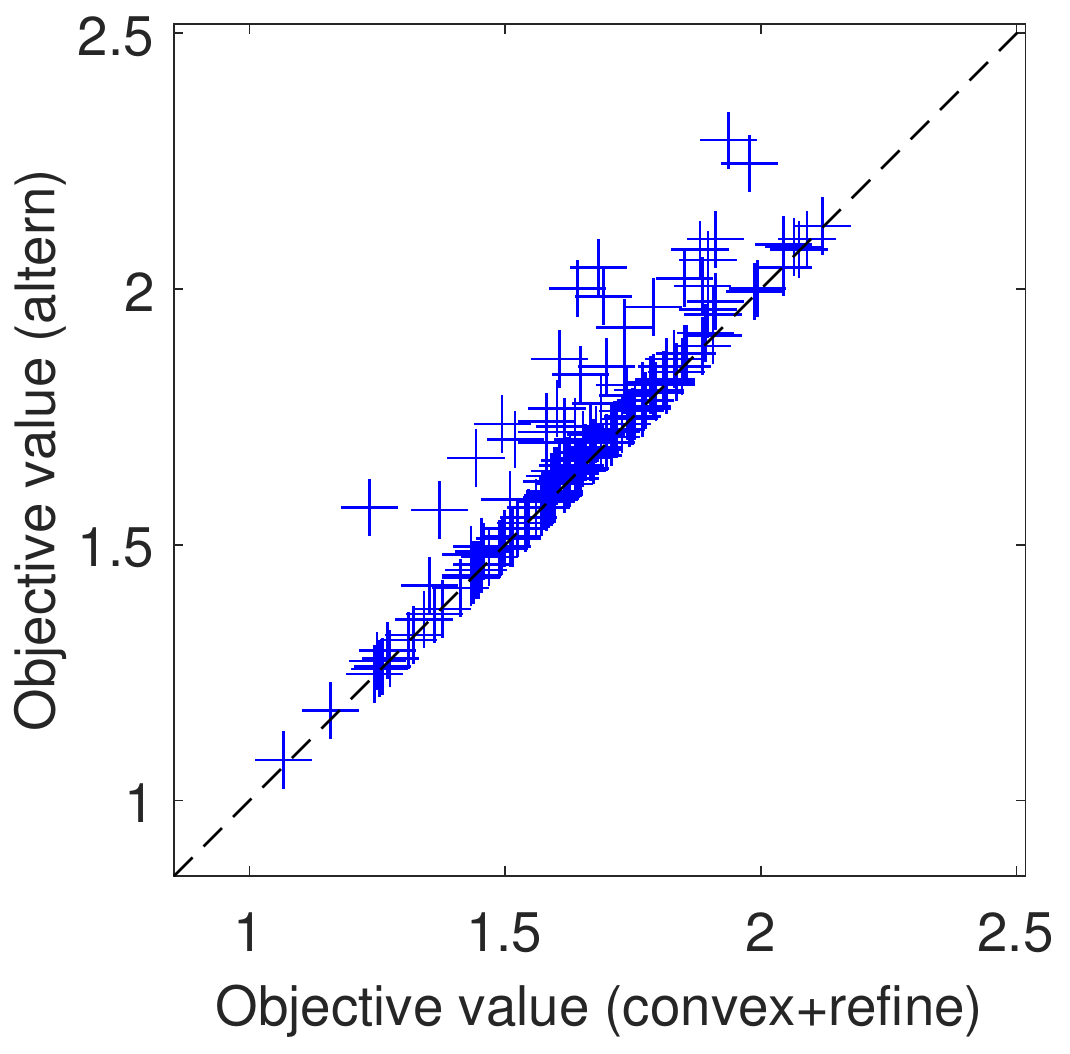} \\ (b)
\end{minipage}
  \caption{Quantitative comparison between the proposed method (``convex+refine") and the alternating minimization (``altern") on the FG3DCar dataset with synthesized outliers. (a) The 2D shape error. (b) The objective value. }\label{fig:fg3dcar-quant}
\end{figure}

Finally, we demonstrate the real applicability of the proposed method with learned landmark detectors. For each landmark, the image patches around the landmark in training images are collected as positive examples and a number of random patches in the background as negative examples. A classifier based on SVM with HOG features is trained with the collected examples. To handle 2D appearance variability, the positive examples are clustered and a mixture of classifiers is learned for each landmark. During testing, a test image is scanned by the learned classifiers and the locations with maximum responses are picked as the landmark locations. The detections are shown in the first column of \refFig{fig:fg3dcar-real}. The false detections are mostly due to self occlusion. The visibility is unknown and we fit the model with all detected landmarks despite of their correctness. The 2D fitted models and 3D reconstructions are shown in the next two columns, which demonstrate that the proposed method can robustly fit the model at the presence of many outliers and without knowing the visibility. The last row shows two failed examples, in which more than half of the landmarks are incorrectly located.

\begin{figure*}
  \centering
  \includegraphics[width=0.9\linewidth]{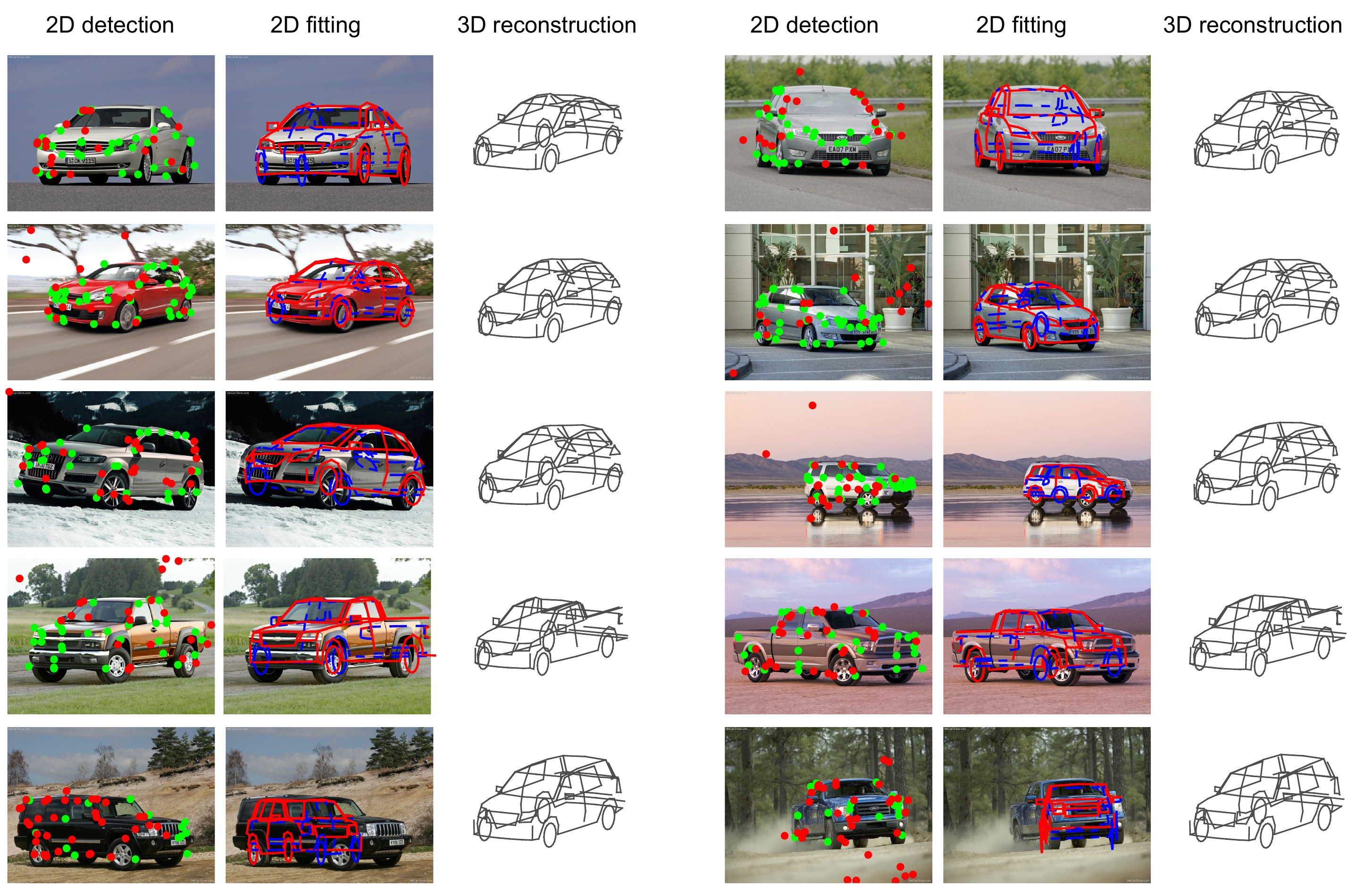}\\
  \caption{Qualitative results on the FG3DCar dataset with detected 2D landmarks. In each example, the detected landmarks superposed on the original image, 2D fitted model and 3D reconstruction are shown. The 2D landmarks are located by learned detectors based on HOG and SVM. The green and red dots correspond to the inliers and outliers, respectively. The outlier is defined as any detected landmark at least 30-pixel away from the manual annotation. Note that the inlier/outlier classification is only for the purpose of illustration and not for model fitting, i.e., the models are fitted with all landmarks despite of the correctness of detection. The last row shows two failed examples.} \label{fig:fg3dcar-real}
\end{figure*}

\subsection{Parameter tuning}
The parameters $\alpha$ and $\beta$ in the proposed model \refEq{eq:finalnoisy} and \refEq{eq:finalrobust} control the sparsity of shape coefficients and outliers, respectively. In general, the best values of $\alpha$ and $\beta$ can be obtained by cross validation for a specific task. But empirically we find that the estimation error changes very smoothly with the parameters varying in a proper range and a set of fixed parameters works well after normalizing the data. More specifically, the coordinates of input 2D shape and 3D basis shapes are scaled such that the average variance of each shape along all directions is unit. With the data normalization, the estimation error as a function of parameters over a subset of CMU motion capture data is shown in \refFig{fig:param}. Each curve has a relatively flat valley. In all experiments in this paper, we fix $\alpha=1$ and $\beta=0.1$ after the data normalization.

\begin{figure}
\begin{minipage}[b]{0.45\linewidth}
  \centering
  \includegraphics[width=\linewidth]{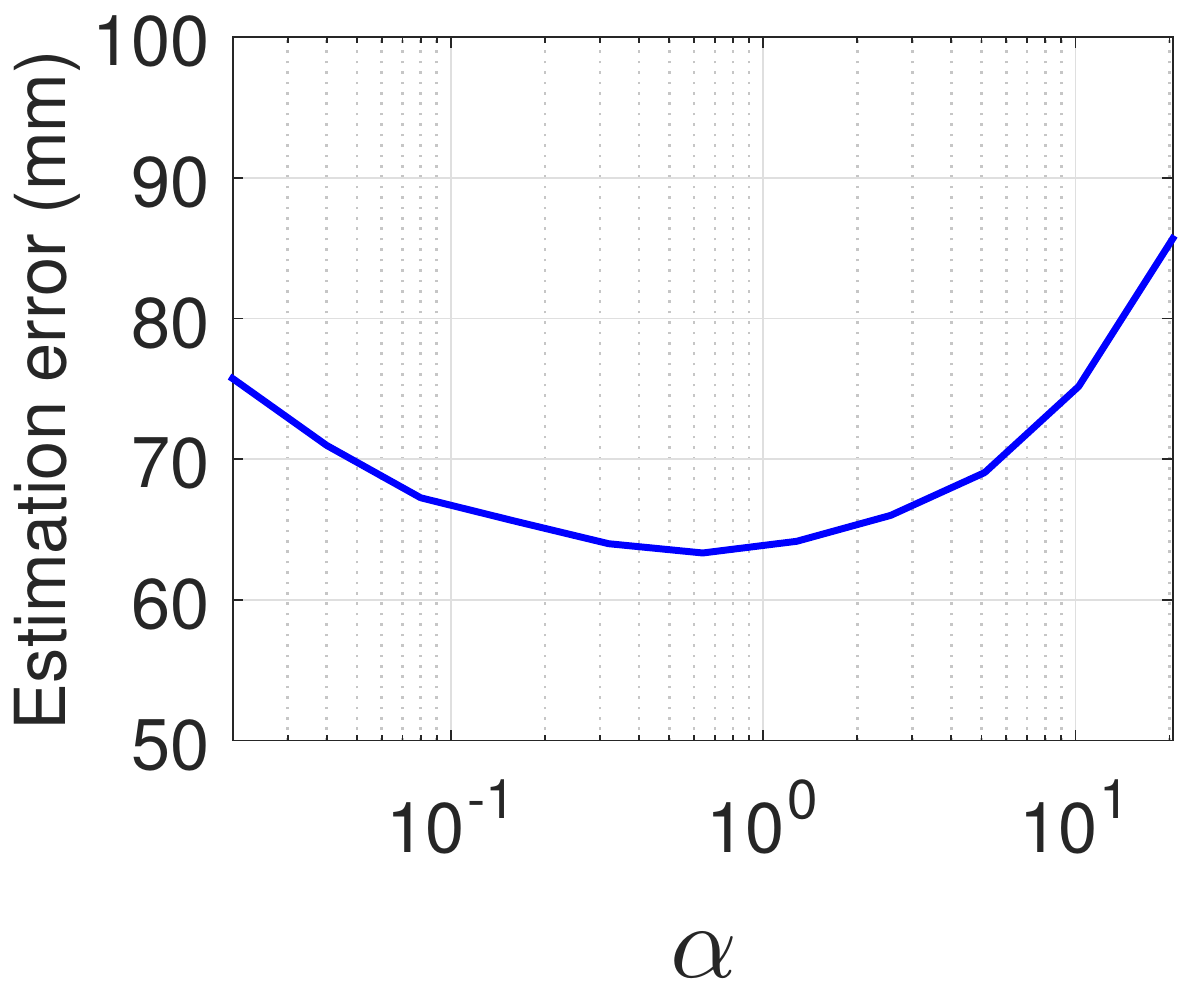}\\(a)
\end{minipage}
\begin{minipage}[b]{0.45\linewidth}
  \centering
  \includegraphics[width=\linewidth]{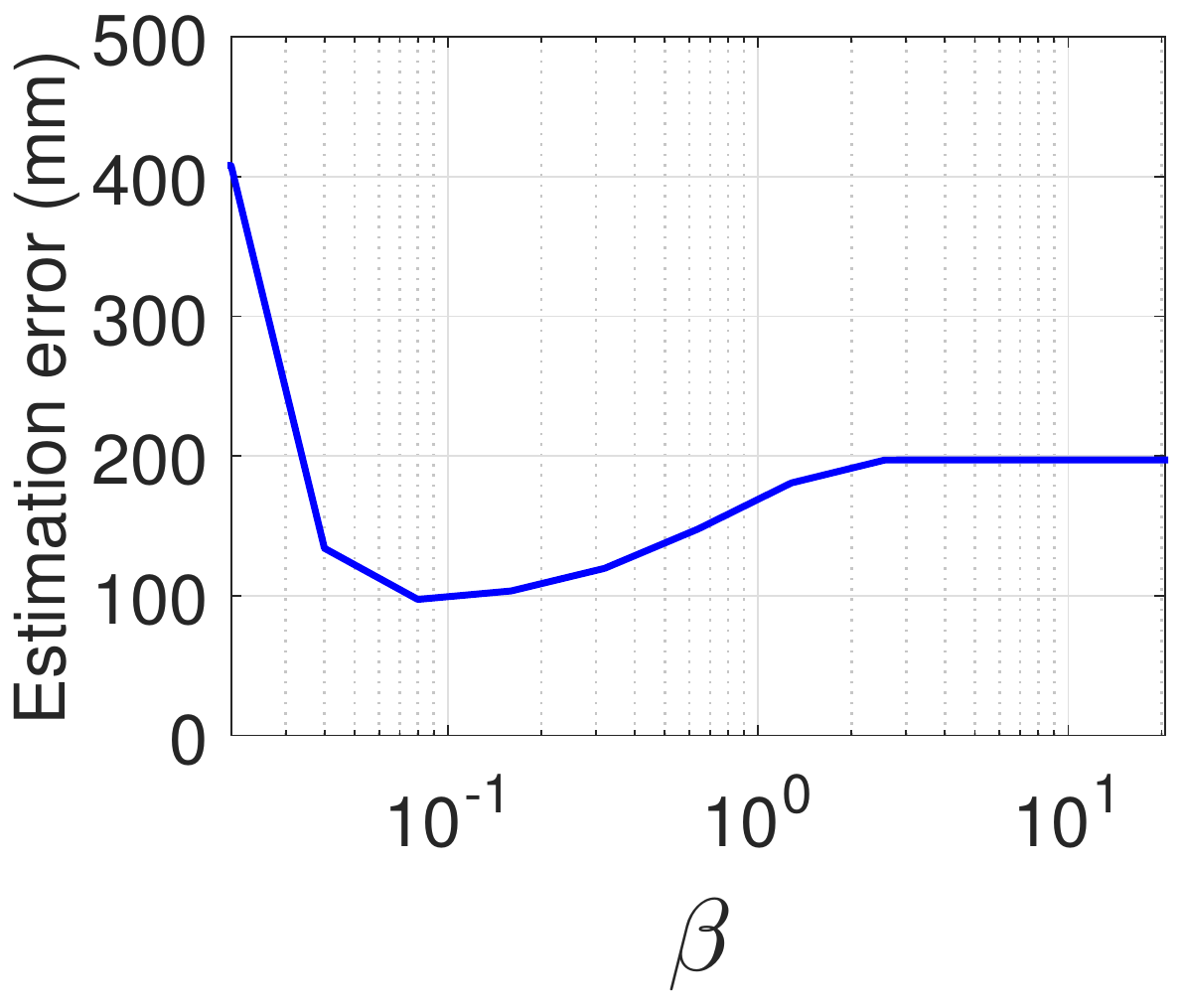}\\(b)
\end{minipage}
  \caption{Sensitivity to model parameters on the CMU motion capture dataset. (a) Estimation error of model \refEq{eq:finalnoisy} as a function of $\alpha$. (b) Estimation error of model \refEq{eq:finalrobust} as a function of $\beta$ with $\alpha=1$ and $20\%$ simulated outliers.} \label{fig:param}
\end{figure}

\subsection{Computational time}

The proposed algorithms are implemented in MATLAB and tested on a desktop with a Intel i7 3.4GHz CPU and 8G RAM. In our experiments, the ADMM algorithm generally converges within 500 iterations to reach a stopping criterion of $10^{-4}$. In human pose estimation, the average computational time of our algorithm is 1.91s per frame, while those of the alternating algorithm and the PMP algorithm \cite{ramakrishna2012reconstructing} are 1.34s and 3.70s, respectively. Note that the for-loops in \refAlg{alg:admm-noisy} (Line 3) can be paralleled to further accelerate the computation.

\section{Discussion}\label{sec:discussion}

In summary, we proposed a method for aligning a 3D deformable shape model with a sparse representation to 2D correspondences by solving a convex program, which guarantees global optimality. Intuitively, we adopted an augmented 3D shape space to achieve a linear representation of both shape and viewpoint variability and proposed to use the spectral-norm regularization to penalize invalid cases caused by the augmentation. We also extended our model to handle outliers in 2D correspondences. We empirically demonstrated the exact recovery property of the relaxed model as well as the advantage of convex optimization compared to alternative nonconvex methods especially when the shape variability was large and there were outliers.

We demonstrated the applicability of the proposed method for reconstructing human poses and car models. The point correspondences across 3D training shapes are provided in the used datasets. With the increasing availability of techniques for 3D shape matching, e.g. \cite{huang2013consistent,huang2014functional}, and databases with rich annotations, e.g. PASCAL3D+ \cite{yu2014beyond} and 3D ShapeNet \cite{shapenet}, we expect to see more applications of the proposed framework to reconstruct a variety of object categories in future.

\appendices

\section{Algorithm to solve the robust model}\label{sec:alg-robust}

The algorithm to solve \refEq{eq:finalrobust} is presented. The problem is rewritten as follows by introducing an auxiliary variable $\bfZ$:
\begin{align}\label{eq:admm-rewrite}
    \min_{\widetilde{\bfM},\bfZ,\bfE,\bfT}~ & \half \left\| \bfW - \bfZ\widetilde{\bfB} -\bfE - \bfT\bfone^T \right\|_{F}^2 \nonumber \\
    & ~~~~ + \alpha \sum_{i=1}^{k}\|\bfM_i\|_2 + \beta\|\bfE\|_1, \nonumber \\
    \st ~& \widetilde{\bfM} = \bfZ.
\end{align}

The augmented Lagrangian of \refEq{eq:admm-rewrite} is
\begin{align}
    \mathcal{L}_{\mu}\left(\widetilde{\bfM},\bfZ,\bfE,\bfT,\bfY\right) &= \half \left\| \bfW - \bfZ\widetilde{\bfB} -\bfE - \bfT\bfone^T \right\|_{F}^2 \nonumber \\
    & + \alpha \sum_{i=1}^{k}\|\bfM_i\|_2 + \beta\|\bfE\|_1 \\
    & + \left<\bfY,\widetilde{\bfM}-\bfZ\right> + \frac{\mu}{2}\left\| \widetilde{\bfM} - \bfZ \right\|_F^2. \nonumber
\end{align}

The following steps are iterated until convergence:
\begin{align}
    &\widetilde{\bfM}^{t+1} = \arg\min_{\widetilde{\bfM}}\mathcal{L}_{\mu}\left(\widetilde{\bfM},\bfZ^t,\bfE^t,\bfT^t,\bfY^t\right); \label{eq:Mstep} \\
    &\bfZ^{t+1} = \arg\min_{\bfZ}\mathcal{L}_{\mu}\left(\widetilde{\bfM}^{t+1},\bfZ,\bfE^t,\bfT^t,\bfY^t\right); \label{eq:Zstep} \\
    &\bfE^{t+1} = \arg\min_{\bfE}\mathcal{L}_{\mu}\left(\widetilde{\bfM}^{t+1},\bfZ^{t+1},\bfE,\bfT^t,\bfY^t\right); \label{eq:Estep} \\
    &\bfT^{t+1} = \arg\min_{\bfT}\mathcal{L}_{\mu}\left(\widetilde{\bfM}^{t+1},\bfZ^{t+1},\bfE^{t+1},\bfT,\bfY^t\right); \label{eq:Tstep} \\
    &\bfY^{t+1} = \bfY^{t} + \mu~\left(\widetilde{\bfM}^{t+1}-\bfZ^{t+1}\right). \label{eq:Ystep}
\end{align}
The whole procedure is very similar to \refAlg{alg:admm-noisy}. The differences are the steps to update $\bfE$ and $\bfT$. For \refEq{eq:Mstep} and \refEq{eq:Zstep}, they can be computed similarly as the steps in \refAlg{alg:admm-noisy}:
\begin{align}
\bfM_i^{t+1} = \mathcal{D}_{\frac{\alpha}{\mu}}(\bfQ_i^t), ~~\forall i\in[1,k],
\end{align}
where $\bfQ_i^t$ is the $i$-th column-triplet of $\bfZ^t - \frac{1}{\mu}\bfY^t -\bfE^t - \bfT^t\bfone^T$,
\begin{align}
\bfZ^{t+1} = & \left((\bfW-\bfE^t-\bfT^t\bfone^T)\widetilde{\bfB}^T+\mu\widetilde{\bfM}^{t+1}+\bfY^t \right) \nonumber \\
& \times \left( \widetilde{\bfB}\widetilde{\bfB}^T+\mu\bfI \right)^{-1}.
\end{align}
For \refEq{eq:Estep}, the problem is a proximal problem associated with the $\ell_1$-norm and can be solved by
\begin{align}
\bfE^{t+1} = \mathcal{S}_{\beta}\left(\bfW-\bfZ^{t+1}\bfB-\bfT^t\bfone^T\right),
\end{align}
where $\mathcal{S}_{\beta}(\bfX)_{ij}=\mbox{sign}(X_{ij})(X_{ij}-\beta)_+$ that refers to the elementwise soft-thresholding operator. For \refEq{eq:Tstep}, the solution is simply given by
\begin{align}
\bfT^{t+1} = \mbox{mean of each row}\left[\bfW-\bfZ^{t+1}\bfB-\bfE^{t+1}\right].
\end{align}

Note that there is no theoretical guarantee for the convergence of ADMM solving a multi-block problem such as \refEq{eq:admm-rewrite} \cite{chen2016direct}, though it always converges in our experiments.

\section{Algorithm to solve dictionary learning}\label{sec:alg-dl}

The algorithm to solve the dictionary learning problem in \refEq{eq:dl} is presented. The cost function can be rewritten as:
\begin{align}
f(\tilde\bfB,\bfC) = \half \| \bfS_j - \sum_{i=1}^{k} C_{ij}\bfB_i \|_{F}^2 + \lambda \sum_{i,j}C_{ij},
\end{align}
where $\tilde\bfB$ is the concatenation of $\bfB_i$s. Then, it is minimized by projected gradient descent and the algorithm is summarized in \refAlg{alg:dl}.

\begin{algorithm}\label{alg:dl}
\LinesNumbered
\caption{Dictionary learning}
\KwIn{$\bfS_1,\cdots,\bfS_n$}
\KwOut{$\bfB_1,\cdots,\bfB_k$}
\vspace{0.5em}
initialize $\tilde\bfB$, $\bfC$, step sizes $\delta_1$ and $\delta_2$\;
\While{not converged}{
\tcc{solve nonnegative sparse coding}
\While{not converged}{
$\bfC \leftarrow \bfC - \delta_1\nabla_{\bfC}f(\tilde\bfB,\bfC)$ \;
\For{$i = 1$ \KwTo $k$}{
\For{$j = 1$ \KwTo $n$}{
\uIf{$C_{ij} < 0$}{
$C_{ij} \leftarrow 0$ \;
}}}}
\tcc{update dictionary}
$\tilde\bfB \leftarrow \tilde\bfB - \delta_2\nabla_{\tilde\bfB}f(\tilde\bfB,\bfC)$ \;
\For{$i = 1$ \KwTo $k$}{
\uIf{$\|\bfB_i\|_F > 1$}{
$\bfB_i \leftarrow \bfB_i/\|\bfB_i\|_F$ \;
}}
}
\end{algorithm}

\vspace{-2em}
\section*{Acknowledgments}
The authors are grateful for support through the following grants:
NSF-DGE-0966142, NSF-IIS-1317788, NSF-IIP-1439681, NSF-IIS-1426840,
ARL MASTCTA W911NF-08-2-0004, ARL RCTA W911NF-10-2-0016, ONR N000141310778
The authors gratefully acknowledge Dr Xiaoyan Hu from Beijing Normal University
for helpful discussions and processing of motion capture data.

\bibliographystyle{IEEEtran}
\bibliography{mybib}

\end{document}

%% file: macros.tex
\newcommand{\R}[1]{\mathbb{R}^{#1}}
\newcommand{\RR}[2]{\mathbb{R}^{#1 \times #2}}

\newcommand{\conv}[1]{\mbox{conv}\left(#1\right)}
\newcommand{\diag}{\mbox{diag}}
\newcommand{\proj}{\mathcal{P}}

\newcommand{\prox}{\mbox{prox}}
\newcommand{\st}{\mbox{s.t.~}}

\newcommand{\refLemma}[1]{Lemma~\ref{#1}}
\newcommand{\refEq}[1]{(\ref{#1})}
\newcommand{\refFig}[1]{Figure~\ref{#1}}
\newcommand{\refProp}[1]{Proposition~\ref{#1}}

\newcommand{\refSec}[1]{Section~\ref{#1}}

\newcommand{\refAlg}[1]{Algorithm~\ref{#1}}

\newcommand{\refApp}[1]{Appendix~\ref{#1}}

\newtheorem{lemma}{Lemma}
\newtheorem{proposition}{Proposition}

\def\bfc{{\boldsymbol{c}}}

\def\bfm{{\boldsymbol{m}}}

\def\bfr{{\boldsymbol{r}}}

\def\bfA{{\boldsymbol{A}}}
\def\bfB{{\boldsymbol{B}}}
\def\bfC{{\boldsymbol{C}}}

\def\bfE{{\boldsymbol{E}}}

\def\bfI{{\boldsymbol{I}}}

\def\bfM{{\boldsymbol{M}}}

\def\bfQ{{\boldsymbol{Q}}}
\def\bfR{{\boldsymbol{R}}}
\def\bfS{{\boldsymbol{S}}}
\def\bfT{{\boldsymbol{T}}}
\def\bfU{{\boldsymbol{U}}}
\def\bfV{{\boldsymbol{V}}}
\def\bfW{{\boldsymbol{W}}}
\def\bfX{{\boldsymbol{X}}}
\def\bfY{{\boldsymbol{Y}}}
\def\bfZ{{\boldsymbol{Z}}}

\def \bfsigma {\boldsymbol{\sigma}}

\def\bfzero{{\boldsymbol{0}}}
\def\bfone{{\boldsymbol{1}}}
\def\half{\frac{1}{2}~}